\documentclass[11pt]{article}
\usepackage[colorlinks=true,bookmarks=false,linkcolor=blue,urlcolor=blue,citecolor=blue,breaklinks=true]{hyperref}

\usepackage[lmargin=3cm,rmargin=3cm,bottom=3cm,top=3cm]{geometry}
\usepackage[numbers]{natbib}
\usepackage{enumitem} % to enumerate with a. b. c. ... to match referee's numbering

\usepackage[utf8]{inputenc} % allow utf-8 input
\usepackage[T1]{fontenc}    % use 8-bit T1 fonts
\usepackage{url}            % simple URL typesetting
\usepackage{booktabs}       % professional-quality tables
\usepackage{amsfonts}       % blackboard math symbols
\usepackage{nicefrac}       % compact symbols for 1/2, etc.
\usepackage{microtype}      % microtypography
\usepackage{graphicx}
\usepackage{amsmath, amsthm, amsfonts, amssymb,tabularx, color}
\usepackage{subfigure}
\usepackage{enumitem}

\newcommand{\Ber}{\text{Ber}}
\newcommand{\Bin}{\text{Bin}}

\newcommand{\bound}{D}

\newcommand{\expect}{\mathbb{E}}
\newcommand{\Rot}{\text{Rot}}
\newcommand{\binomquantestimator}{\hat{\bar{X}}_{\pi_{sk}(\Bin(m,p))}}
\newcommand{\binomquantui}{U}
\newcommand{\binomnoise}{T}
\newcommand{\binomprotocol}{\pi_{sk}(\Bin(m,p))}
\newcommand{\quantkprotocol}{\pi_{sk}}
\newcommand{\dme}{DME }
\newcommand{\dmecap}{DME }

\newcommand{\binomdelta}{\Delta}

\def\normal{{\mathcal N}}

\def\M{{\mathcal M}}
\def\O{{\mathcal O}}

\def\D{{\mathcal D}}

\def\reals{{\mathbb R}}
\def\integers{{\mathbb Z}}

\def\norm#1{\mathopen\| #1 \mathclose\|}% use instead of $\|x\|$

\newcommand{\ignore}[1]{}

\newcommand\var{\text{Var}}

\def\bold0{\mathbf{0}}

\def\epsilon{\varepsilon}

\newcommand{\defeq}{\triangleq}

\newtheorem{theorem}{Theorem}
\newtheorem{lemma}{Lemma}

\newtheorem{claim}{Claim}
\newtheorem{corollary}{Corollary}

\newtheorem{definition}{Definition}

\newcommand{\cO}{\mathcal{O}}

\newcommand{\cX}{\mathcal{X}}
\newcommand{\cE}{\mathcal{E}}
\newcommand{\cC}{\mathcal{C}}

\newcommand{\namedref}[2]{\mbox{\hyperref[#2]{#1~\ref*{#2}}}}

\newcommand{\figurerefb}[2]{\mbox{\hyperref[#1]{Figure~\ref*{#1}#2}}}

\newcommand{\equationref}[1]{\mbox{\hyperref[#1]{(\ref*{#1})}}}
\renewcommand{\eqref}{\equationref}
\newcommand{\ed}{\stackrel{\mathrm{def}}{=}}

\newcommand{\EE}{\mathbb{E}}
\newcommand{\II}{\mathbb{I}}
\newcommand{\RR}{\mathbb{R}}

\title{cpSGD: Communication-efficient and differentially-private distributed SGD}

\author{
  Naman Agarwal \footnote{Most of the research was performed while the author was on an internship at Google Research, New York} \\
  Princeton University\\
  \texttt{namana@cs.princeton.edu} \\
  \and
  Ananda Theertha Suresh \\
  Google Research, New York \\
  \texttt{theertha@google.com} \\
  \and
  Felix Yu \\
  Google Research, New York \\
  \texttt{felixyu@google.com} \\
  \and
  Sanjiv Kumar \\
  Google Research, New York \\
  \texttt{sanjivk@google.com} \\
  \and
  H. Brendan Mcmahan \\
  Google Research, Seattle \\
  \texttt{mcmahan@google.com} \\
}

\begin{document}
% \nipsfinalcopy is no longer used

\maketitle

\begin{abstract}
 Distributed stochastic gradient descent is an important subroutine in distributed learning. A setting of particular interest is when the clients are mobile devices, where two important concerns are communication efficiency and the privacy of the clients. Several recent works have focused on reducing the communication cost or introducing privacy guarantees, but none of the proposed communication efficient methods are known to be privacy preserving and none of the known privacy mechanisms are known to be communication efficient. To this end, we study algorithms that achieve both communication efficiency and differential privacy. For $d$ variables and $n \approx d$ clients, the proposed method uses $\cO(\log \log(nd))$ bits of communication per client per coordinate and ensures constant privacy.

We also extend and improve previous analysis of the \emph{Binomial mechanism} showing that it achieves nearly the same utility as the Gaussian mechanism, while requiring fewer representation bits, which can be of independent interest.

\end{abstract}

\section{Introduction}
\subsection{Background}

Distributed stochastic gradient descent (SGD) is a basic building block of modern machine learning~\citep{McdonaldGM10,  dean2012large, coates2013deep, PoveyZK14,abadi2016tensorflow, mcmahan2016federated, AlistarhLTV16}. In the typical scenario of
synchronous distributed learning, in every round, each client obtains a copy of a global model which it updates based on its local data.  The updates (usually in the form of gradients) are sent to a parameter server, where they are averaged and used to update the global model. Alternatively, without a central server, each client saves a global model, broadcasts the gradient to all other clients, and updates its model with the aggregated gradient.

Often, the communication cost of sending the gradient becomes the bottleneck~\citep{recht2011hogwild, li2014communication, li2014scaling}. 
To address this issue, several recent works have focused on reducing the communication cost of distributed learning algorithms via gradient quantization and sparsification~\citep{seide20141, gupta2015deep, suresh2016distributed, konevcny2016federated, konevcny2016randomized, alistarh2017communication, wen2017terngrad}. These algorithms have been shown to improve communication cost and hence communication time in distributed learning.  This is especially effective in 
%client based distributed synchronous learning algorithms 
the federated learning setting where clients are mobile devices with expensive up-link communication cost~\citep{fedlearning, konevcny2016federated}. 

While communication is a key concern in client based distributed machine learning, an equally important consideration is that of protecting the privacy of participating clients and their sensitive information. Providing rigorous privacy guarantees for machine learning applications has been an area of active recent interest~\citep{bassily2014private,wu2017bolt,sarwate2013signal}. Differentially private gradient descent algorithms in particular were studied in the work of~\citep{abadi2016deep}. A direct application of these mechanisms in distributed settings leads to algorithms with high communication costs. The key focus of our paper is to analyze mechanisms that achieve rigorous privacy guarantees as well as have communication efficiency.

\subsection{Communication efficiency}

We first describe synchronous distributed SGD formally. Let $F(w) : \reals^d \rightarrow \reals$ be of the form 
\[ F(w) = \frac{1}{M} \cdot \sum^M_{i=1} f_i(w),\]
where each $f_i$ resides at the $i^{th}$ client. For example, $w$'s are weights of a neural network and $f_i(w)$ is the loss of the network on data located on client $i$. Let $w^0$ be the initial value. At round $t$, the server transmits $w^t$ to all the clients and asks a random set of $n$ (batch size / lot size) clients to transmit their local gradient estimates ${g}^t_i(w^t)$. Let $S$ be the subset of clients. 

The server updates as follows 
\[ g^t(w^t) = \frac{1}{n}\sum_{i \in S} {g}^t_i(w^t),
\qquad \qquad w^{t+1} \defeq w^t - \gamma g^t(w^t)\]
for some suitable choice of $\gamma$.  Other optimization algorithms such as momentum, Adagrad, or Adam can also be used instead of the SGD step above.

Naively for the above protocol, each of the $n$ clients needs to transmit $d$ reals, typically using $\cO(d \cdot \log 1/\eta)$ bits\footnote{$\eta$ is the per-coordinate quantization accuracy. To represent a $d$ dimensional vector $X$ to an constant accuracy in Euclidean distance, each coordinate is usually quantized to an accuracy of $\eta = 1/\sqrt{d}$.}.

This communication cost can be prohibitive, e.g., for a medium size PennTreeBank language model~\citep{zaremba2014recurrent}, the number of parameters $d > 10\text{ million}$ and hence total cost is $\sim 38 $MB (assuming 32 bit float), which is too large to be sent from a mobile phone to the server at every round. 

Motivated by the need for communication efficient protocols, various quantization algorithms have been proposed to reduce the communication cost \citep{suresh2016distributed,konevcny2016federated, konevcny2016randomized, tsuzuku2018variancebased, lin2018deep, wen2017terngrad, AlistarhLTV16}. In these protocols, the clients quantize the gradient by a function $q$ and send an efficient representation of $q({g}^t_i(w^t))$ instead of its actual local gradient ${g}^t_i(w^t)$.  The server computes the gradient as 
\[
\tilde{g}^t(w^t) = \frac{1}{n}\sum_{i \in S} q({g}^t_i(w^t)),
\]
and updates $w^t$ as before. Specifically, \cite{suresh2016distributed} proposes a quantization algorithm which reduces the requirement of full (or floating point) arithmetic precision to a bit or few bits per value on average. There are many subsequent works e.g., see~\cite{konevcny2016randomized} and in particular \cite{AlistarhLTV16} showed that stochastic quantization and Elias coding~\citep{elias1975} can be used to obtain communication-optimal SGD for convex functions. If the expected communication cost at every round $t$ is bounded by $c$, then the total communication cost of the modified gradient descent is  at most
\begin{equation}
\label{eq:com}
T \cdot c.
\end{equation}
All the previous papers relate the error in gradient compression to SGD convergence. We first state one such result for completeness for non-convex functions and prove it in Appendix~\ref{appsec:biasedsgd}. Similar (and stronger) results can be obtained for (strongly) convex functions using results in~\cite{ghadimi2013stochastic} and~\cite{rakhlin2012making}. 
\begin{corollary}[\cite{ghadimi2013stochastic}]
\label{cor:qsgd}
%Suppose the assumptions in Lemma~\ref{lem:sgd} hold. 
Let $F$ be $L$-smooth 
%i.e. there exists an $L$ such that 
%\[ \forall \;\;x,y\;\;\|\nabla F(x) - \nabla F(y)\|_2 \leq L\|x - y\|_2.
%\]
and $\forall x\; \|\nabla F(x)\|_2 \leq \bound$. 
Let $w^0$ satisfy $F(w^0) - F(w^*) \leq D_F$. Let $q$ be a quantization scheme, and 
$\gamma \defeq \min\left\{L^{-1}, \sqrt{2D_F}(\sigma \sqrt{LT})^{-1}\right\},$ then after $T$ rounds

\[
\expect_{t \sim (\text{Unif}[T])}[ \|\nabla F(w^t)\|^2_2] \leq \frac{2D_FL}{T} + \frac{2\sqrt{2}\sigma\sqrt{LD_F}}{\sqrt{T}} + \bound B,
\]
\begin{equation}
 \label{eq:gvarq}
 \text{ where }\qquad
\sigma^2 = \max_{1\leq t \leq T} 2\expect[\|g^t(w^t) - \nabla F(w^t)\|^2_2]  +  2\max_{1\leq t \leq T} \expect_q[\|{g}^t(w^t) - \tilde{g}^t(w^t)\|^2_2],
 \end{equation} 
 and %$B$ is given by
 $ B = \max_{1\leq t \leq T} \|\expect_q[{g}^t(w^t) - \tilde{g}^t(w^t)]\|$. The expectation in the above equations is over the randomness in gradients and quantization.
\end{corollary} 
The above result relates the convergence of distributed SGD for non-convex functions to the \emph{worst-case} mean square error (MSE) and bias in gradient mean estimates in Equation~\eqref{eq:gvarq}. Thus smaller the mean square error  in gradient estimation, better convergence. Hence, we focus on the problem of distributed mean estimation (DME), where the goal is to estimate the mean of a set of vectors.
\subsection{Differential privacy}
While the above schemes reduce the communication cost, it is unclear what (if any) privacy guarantees they offer.  We study privacy from the lens of differential privacy (DP). The notion of differential privacy \citep{dwork2006calibrating} provides a strong notion of individual privacy while permitting useful data analysis in machine learning tasks. We refer the reader to \citep{dworkroth} for a survey.  Informally, for the output to be differentially private, the estimated model should be indistinguishable whether a particular client's data was taken into consideration or not. We define this formally in Section \ref{sec:overall_diff_privacy}.

In the context of client based distributed learning, we are interested in the privacy of the gradients aggregated from clients; differential privacy for the average gradients implies privacy for the resulting model since DP is preserved by post-processing. 

The standard approach is to let the  server add the noise to the averaged gradients (e.g., see~\cite{dworkroth, abadi2016deep} and references within). However, the above only works under a restrictive assumption that the clients can trust the server. Our goal is to also minimize the need for clients to trust the central aggregator, and hence we propose the following model:

\emph{Clients add their share of the noise to their gradients $g^t_i$ before transmission. Aggregation of gradients at the server results in an estimate with noise equal to the sum of the noise added at each client.}

This approach improves over server-controlled noise addition in several scenarios:

%\begin{itemize}[leftmargin=*]
\textbf{Clients do not trust the server}: Even in the scenario when the server is not trustworthy, the above scheme can be implemented via cryptographically secure aggregation schemes~\citep{bonawitz2016practical}, which ensures that the only information about the individual users the server learns is what can be inferred from the sum. Hence, differential privacy of the aggregate now ensures that the parameter server does not learn any individual user information. This will encourage clients to participate in the protocol even if they do not fully trust the server. We note that while secure aggregation schemes add to the communication cost (e.g., \cite{bonawitz2016practical} adds $\log_2(k \cdot n)$ for $k$ levels of quantization), our proposed communication benefits still hold. For example, if $n = 1024$, a 4-bit quantization protocol would reduce communication cost by 67\% compared to the 32 bit representation.

\textbf{Server is negligent, but not malicious}: the server may "forget" to add noise, but is not malicious and not interested in learning characteristics of individual users. However, if the server releases the learned model to public, it needs to be differentially-private.

A natural way to extend the results of~\cite{dworkroth,abadi2016deep} is to let individual users add Gaussian noise to their gradients before transmission. Since the sum of Gaussians is Gaussian itself, differential privacy results follow. However, the transmitted values now are real numbers and the benefits of gradient compression are lost. Further, secure aggregation protocols~\cite{bonawitz2016practical} require discrete inputs. To resolve these issues, we propose that the clients add noise drawn from an appropriately parameterized Binomial distribution. We refer to this as the \emph{Binomial mechanism}. Since Binomial random variables are discrete, they can be transmitted efficiently. Furthermore, 
the choice of the Binomial is convenient in the distributed setting because sum of Binomials is also binomially distributed i.e., if 
\[
Z_1,Z_2 \sim \Bin(N_1, p) \;\;\text{ then }\;\; Z_1 + Z_2 \sim \Bin(N_1 + N_2, p).
\]
Hence the total noise post aggregation can be analyzed easily, which is convenient for the distributed setting\footnote{Another choice is the Poisson distribution. Different from Poisson, the Binomial distribution has bounded support and has an easily analyzable communication complexity which is always bounded.}.
Binomial mechanism can be of independent interest in other applications with discrete output as well. Furthermore, unlike Gaussian it avoids floating point representation issues. 

\subsection{Summary of our results}
\textbf{Binomial mechanism}: We first study Binomial mechanism as a generic mechanism to release discrete valued data. Previous analysis of the Binomial mechanism (where you add noise $\Bin(N,p)$) was due to~\cite{dwork2006our}, who analyzed the $1$-dimensional case for $p = 1/2$ and showed that to achieve $(\epsilon,\delta)$ differential privacy, $N$ needs to be $\geq 64 \log(2/\delta)/\epsilon^2$. We improve the analysis in the following ways: 
\begin{itemize}[leftmargin=*]
\item \textbf{$d$-dimensions.} We extend the analysis of $1$-dimensional Binomial mechanism to $d$ dimensions. Unlike the Gaussian distribution, Binomial is not rotation invariant making the analysis more involved. The key fact utilized in this analysis is that Binomial distribution is locally rotation-invariant around the mean. 
\item \textbf{Improvement.}
 We improve the previous result and show that $N \geq 8 \log(2/\delta)/\epsilon^2$ suffices for small $\epsilon$, implying that the Binomial and Gaussian mechanism perform identically as $\epsilon \to 0$. We note that while this is a constant improvement , it is crucial in making differential privacy practical.
\end{itemize}

\textbf{Differentially-private distributed mean estimation (DME)}: A direct application of Gaussian mechanism requires $n \cdot d$ reals and hence $n \cdot d  \cdot \log (n d)$ bits of communication. This can be prohibitive in practice. However, a direct application of quantization~\cite{suresh2016distributed} and Binomial mechanism has high communication cost. We show that random rotation together with the notion of \emph{high probability sensitivity} can significantly improve communication.

In particular, for $\epsilon=O(1)$, we provide an algorithm achieving equal privacy and error as that of the Gaussian mechanism with communication 
\[\leq n \cdot d \cdot \left(\log_2 \left( 1+\frac{d}{n}\right )
+ \cO(\log \log \left( \frac{nd}{\delta} \right) \right) \text{ bits,}
\]
per round of distributed SGD. Hence when $d \approx n$, the number of bits is $n \cdot d \cdot \log (\log (nd)/\delta)$. 

The rest of the paper is organized as follows. In Section~\ref{sec:overall_diff_privacy}, we review the notion of differential privacy and state our results for the Binomial mechanism. Motivated by the fact that the convergence of SGD can be reduced to the error in gradient estimate computation per-round, we 
formally describe the problem of \dme in Section~\ref{sec:distmean} and state our results in Section~\ref{sec:results}.

In Section \ref{sec:BINdescription}, we provide and analyze the implementation of the binomial mechanism in conjunction with quantization in the context of DME. The main idea is for each client to add noise drawn from an appropriately parameterized Binomial distribution to each quantized value before sending to the server. The server further subtracts the bias introduced by the noise to achieve an unbiased mean estimator. We further show in Section \ref{sec:rotation} that the rotation procedure proposed in \cite{suresh2016distributed} which reduces the MSE is helpful in reducing the additional error due to differential privacy. 

\section{Differential privacy}
\label{sec:overall_diff_privacy}
\subsection{Notation}
We start by defining the notion of differential privacy. Formally, given a set of data sets $\D$  provided with a notion of neighboring data sets $\mathcal{N}_{\D} \subset \D \times \D$ and a query function $f : \D \rightarrow \cX$, a mechanism $\M : \cX \rightarrow \O$ to release the answer of the query, is defined to be $(\epsilon, \delta)$ differentially private if for any measurable subset $S \subseteq \O$ and two neighboring data sets $(D_1, D_2) \in \mathcal{N}_{\D}$,
\begin{equation} \Pr\left(\M(f(D_1)) \in S\right) \leq e^{\epsilon}\Pr \left(\M(f(D_2)) \in S\right) + \delta.
\end{equation}
Unless otherwise stated, for the rest of the paper, we will assume the output spaces $\cX,\cO \subseteq \reals^d$. We consider the mean square error as a metric to measure the error of the mechanism $\M$. Formally,
\[
\cE(\M) \defeq \max_{D \in \D} \EE[\norm{\M(f(D)) - f(D)}^2_2].
\]
A key quantity in characterizing differential privacy for many mechanisms is the sensitivity of a query $f: \D \rightarrow \reals^d$ in a given norm $\ell_q$. Formally this is defined as
\begin{equation}
\label{eqn:deltaqdefn}  
\Delta_q \defeq \max_{(D_1, D_2) \in \mathcal{N}_{\D}} \norm{f(D_1) - f(D_2)}_q.
\end{equation}
The canonical mechanism to achieve $(\epsilon, \delta)$ differential privacy is the Gaussian mechanism $\M_{g}^{\sigma}$~\cite{dworkroth}:
\[
\M_g^{\sigma}(f(D)) \defeq f(D) + Z,
\]
where $Z \sim \normal(0,\sigma^2 \II_d)$. We now state the well-known privacy guarantee of the Gaussian mechanism. 
\begin{lemma}[~\cite{dworkroth}]
\label{lem:Gaussian_overall}
For any $\delta$, $\ell_2$ sensitivity bound $\Delta_2$, and $\sigma$ such that \[
\sigma \geq \Delta_2 \sqrt{2 \log \frac{1.25}{\delta}},
\]
$\M_{g}^{\sigma}$ is $(\frac{\Delta_2 }{\sigma} \sqrt{2 \log\frac{1.25}{\delta}},\delta)$ differentially private
\footnote{All logs are to base $e$ unless otherwise stated.} and the error is bounded by $
d \cdot \sigma^2.$
\end{lemma}

\subsection{High probability sensitivity}

In this section we introduce a notion of high probability sensitivity which allows us to work randomized queries which do not have a worst case bound on sensitivity but have bounded sensitivity with high probability.  
Let $Q \defeq \{q_i \in \mathbb{N}\}$ represent a set of natural numbers and $\Delta_Q \defeq \{ \Delta_{q_i}\}$ represent a subset of real numbers. For two random vectors $v_1, v_2$, the event $\|v_1 - v_2\|_Q \leq \Delta_Q$ is defined as
\[ (\|v_1 - v_2\|_Q \leq \Delta_Q) \defeq \bigcup_{i} \;\;(\|v_1 - v_2\|_{q_i} \leq \Delta_{q_i})\]
\begin{definition}[$(\Delta_Q, \delta)$ sensitivity]
\label{defn:highprobsense}
  Given a set of integers $Q$ and values $\Delta_Q, \delta$, we call a randomized function $f: \mathcal{D} \rightarrow \mathcal{X}$, $(\Delta_Q, \delta)$ sensitive, if for any two neighboring data sets $D_1, D_2 \in \mathcal{N}_\mathcal{D}$, there exist coupled random variables $X_1, X_2 \in \mathcal{X}$ such that the marginal distributions of $X_1, X_2$ are identical to that of ${f}(D_1)$ and ${f}(D_2)$ and 
\begin{equation}
\label{eqn:couplingcond}
 \Pr_{X_1, X_2}(\|X_1 - X_2\|_Q \leq \Delta_Q) \geq 1 - \delta.
\end{equation}
\end{definition}
We show the following result for high-probability sensitivity and 
the proof is provided in  Appendix~\ref{sec:highprobsensitivity}.
\begin{lemma}
\label{lemma:highprobsensitivity}
  Let $\mathcal{M}: \mathcal{X} \rightarrow \mathcal{O}$ be an $(\epsilon,\delta)$ differentially private mechanism for sensitivity $\Delta_Q$ and let  ${f}:\mathcal{D} \rightarrow \mathcal{X}$ be a $(\Delta_Q, \delta')$ sensitive function. Then the composed mechanism ${\mathcal{M}}(f(D))$ is $(\epsilon, \delta + \delta')$ differentially private. 
\end{lemma}

\section{Binomial Mechanism}
 We now define the Binomial mechanism for the case when the output space $\cX$ of the query $f$ is $\mathbb{Z}^d$. The Binomial mechanism is parameterized by three quantities $N,p,s$ where $N \in \mathbb{N}, p \in (0,1)$, and quantization scale $s = 1/j$ for some $j \in \mathbb{N}$ and is given by
\begin{equation}
\label{eqn:binommechdefn}
  \M_b^{N,p,s}(f(D)) \defeq f(D) + (Z - Np)\cdot s,
\end{equation}
where for each coordinate $i$, $Z_i \sim \Bin(N,p)$ and independent.
One dimensional binomial mechanism was introduced by~\cite{dwork2006our} for the case when $p = 1/2$. We analyze the mechanism for the general $d$-dimensional case and for any $p$. This analysis is involved as the Binomial mechanism is not rotation invariant. By carefully exploiting the local rotation invariant structure near the mean, we show that:
% \fe{Sanjiv: we need to highlight why this is not trivial.}
\begin{theorem}
\label{thm:discbinomial}
For any $\delta$, parameters $N, p, s$ and sensitivity bounds $\Delta_1, \Delta_2, \Delta_{\infty}$ such that
\[
Np(1-p) \geq \max \left(23\log(10d/\delta), 2 \Delta_\infty / s \right),
\]
the Binomial mechanism is $(\epsilon, \delta)$ differentially private for
\begin{equation}
\label{eqn:binomialguarantee}
\epsilon = \frac{\Delta_2\sqrt{2 \log \frac{1.25}{\delta}}}{s\sqrt{Np(1-p)}} 
+ \frac{\Delta_2c_p \sqrt{\log \frac{10}{\delta}} + \Delta_1 b_p}{sNp(1-p)(1-\delta/10)}
+  \frac{\frac{2}{3}\Delta_\infty  \log \frac{1.25}{\delta}
+\Delta_\infty d_p \log \frac{20d}{\delta} \log \frac{10}{\delta}}{sNp(1-p)},
\end{equation}
where $b_p$, $c_p,$ and $d_p$ are defined in ~\eqref{eq:bp}, ~\eqref{eq:cp}, and ~\eqref{eq:dp} respectively, and for $p=1/2$, $b_p = 1/3$, $c_p = 5/2$, and $d_p = 2/3$. The error of the mechanism is 
\[
d  \cdot s^2 \cdot N p (1-p).
\]
\end{theorem}
The proof is given in Appendix~\ref{app_sec:binomial}.
We make some remarks regarding the design and the guarantee for the Binomial Mechanism. Note that the privacy guarantee for the Binomial mechanism depends on all three sensitivity parameters $\Delta_2, \Delta_{\infty}, \Delta_{1}$ as opposed to the Gaussian mechanism which only depends on $\Delta_2$. 
%Since the $\ell_1$ norm can be a factor of $\sqrt{d}$ larger than $\ell_2$, 
The $\Delta_1$ and $\Delta_\infty$ terms can be seen as the added complexity due to discretization.

Secondly setting $s=1$ (i.e. providing no scale to the noise) in the expression \eqref{eqn:binomialguarantee}, it can be readily seen that the terms involving $\Delta_1$ and $\Delta_2$ scale differently with respect to the variance of the noise. This motivates the use of the accompanying quantization scale $s$ in the mechanism. Indeed it is possible that the resolution of the integer that is provided by the Binomial noise could potentially be too large for the problem leading to worse guarantees. In this setting, the quantization parameter $s$ helps normalize the noise correctly. Further, it can be seen as long as the variance of the random variable $s\cdot Z$ is fixed, increasing $Np(1-p)$ and decreasing $s$ makes the Binomial mechanism closer to the Gaussian mechanism. Formally, if we let $\sigma = s \sqrt{N p(1-p)}$ and $s \leq \sigma/( c \sqrt{ d})$, then using the Cauchy-Schwartz inequality, the $\epsilon$ guarantee \eqref{eqn:binomialguarantee} can be rewritten as 
\[
\epsilon =  {\Delta_2}/{\sigma}\sqrt{2\log {1.25}/{\delta}}  \left(1 + \cO\left({1}/{c} \right) \right).
\]

The variance of the Binomial distribution is $Np(1-p)$ and the leading term in $\epsilon$ matches exactly
 %\footnote{up to a small discrepancy in the $\log$}
 the $\epsilon$ term in Gaussian mechanism. 
Furthermore, if $s$ is $o(1/\sqrt{d})$, then this mechanism is very similar to the Gaussian mechanism. This result agrees with the Berry-Esseen type Central limit theorems for the convergence of one dimensional Binomial distribution to the Gaussian distribution.

  In Figure \ref{fig:binomialvsGaussian}, we plot the error vs $\epsilon$ for Gaussian and Binomial mechanism. Observe that as scale is reduced, error vs privacy trade-off for Binomial mechanism approaches that of Gaussian mechanism. 

\begin{figure}[tb]\centering
   \begin{minipage}{0.32\textwidth}
     \includegraphics[width=\linewidth]{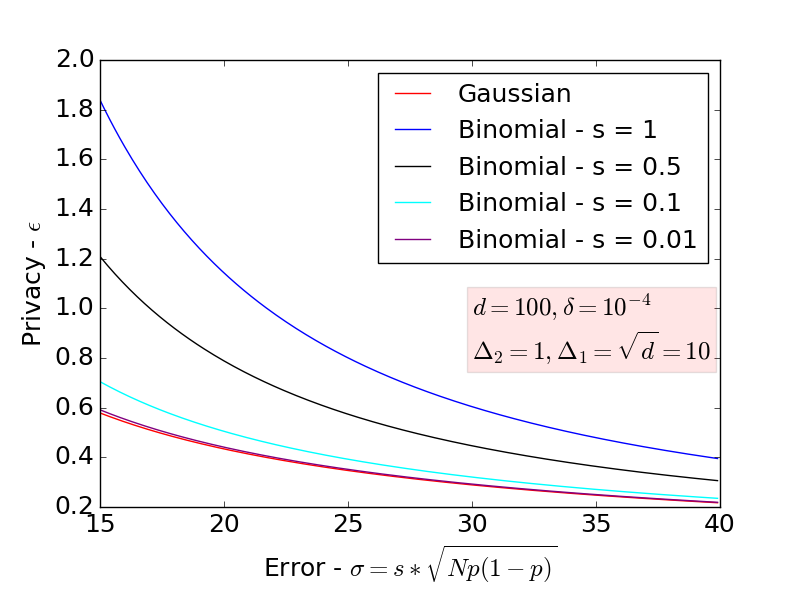}
     \caption{Comparison of error vs privacy for Gaussian and Binomial mechanism at different scales}\label{fig:binomialvsGaussian}
   \end{minipage}
   \hspace{+0.5cm}
   \begin {minipage}{0.6\textwidth}
     \subfigure[$\epsilon=4.0$]{\includegraphics[width=0.45\textwidth]{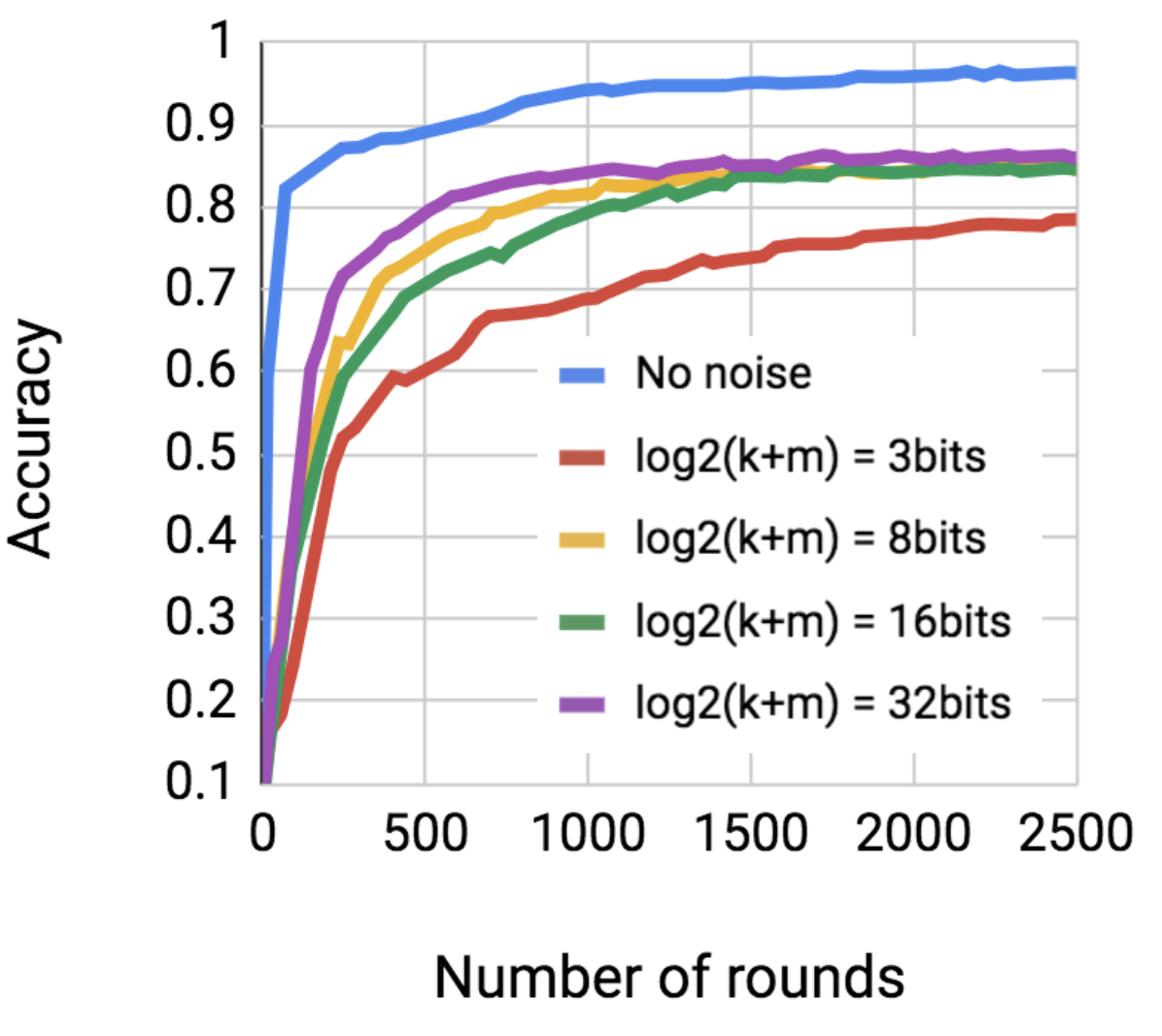}}
\subfigure[$\epsilon=2.0$]{\includegraphics[width=0.45\textwidth]{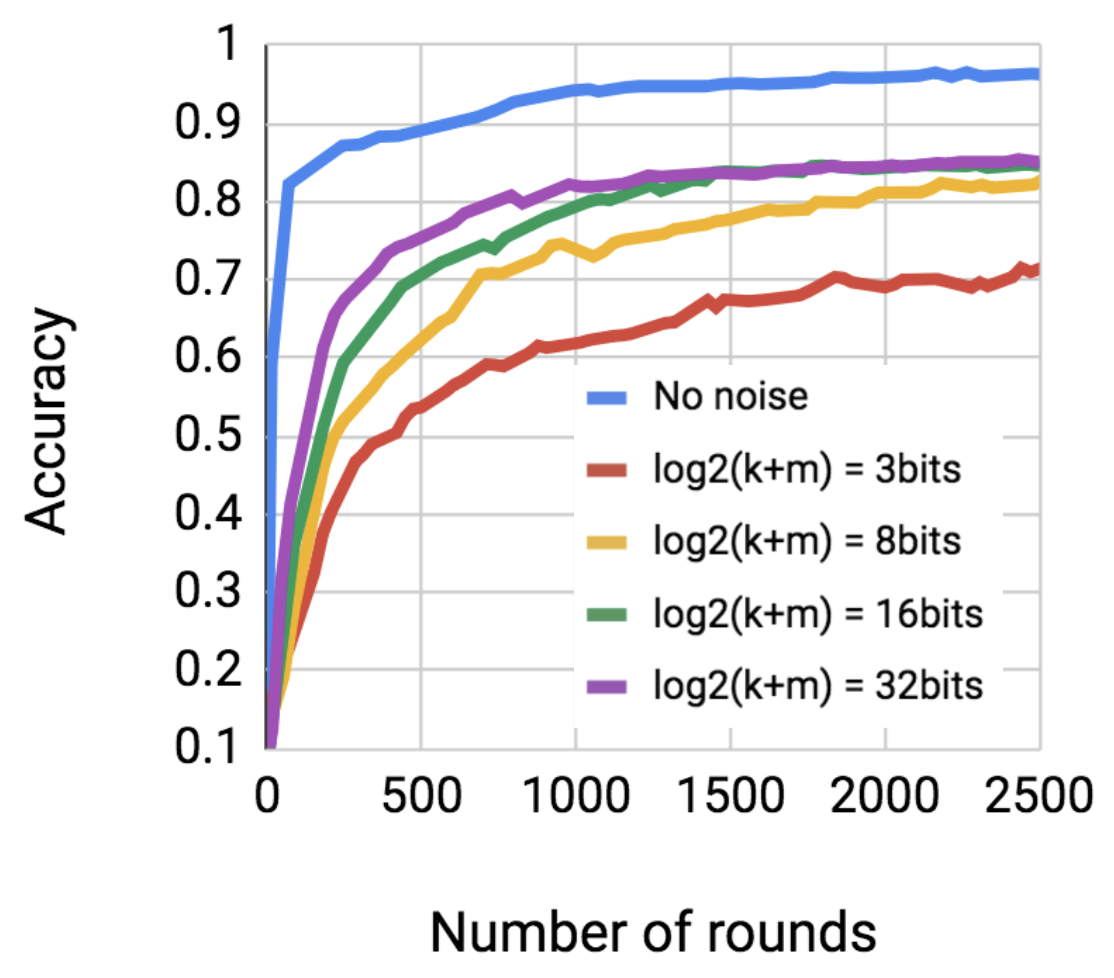}}
\caption{cpSGD with rotation on the  infinite MNIST dataset. $k$ is the number of quantization levels, and $m$ is the parameter of the binomial noise ($p = 0.5$, $s = 1$).
The baseline is without quantization and differential privacy. $\delta = 10^{-9}$.}
\label{fig:experiment}
\end{minipage}
\end{figure}

\section{Distributed mean estimation (DME)}
\label{sec:distmean}
We have related the SGD convergence rate to the MSE in approximating the gradient at each step in Corollary~\ref{cor:qsgd}. Eq.~\eqref{eq:com} relates the communication cost of SGD to the communication cost of estimating gradient means. Advanced composition theorem (Thm. 3.5 \cite{kairouz2017composition}) or moments accounting~\cite{abadi2016deep} can be used to relate the privacy guarantee of SGD to that of gradient mean estimate at each instance $t$. We also note that in SGD, we often sample the clients, standard privacy amplification results via sampling~\cite{abadi2016deep}, can be used to get tighter bounds in this case. 

Therefore, akin to \cite{suresh2016distributed}, in the rest of the paper we just focus on the MSE and privacy guarantees of DME. The results for synchronous distributed GD follow from Corollary~\ref{cor:qsgd} (convergence), advanced composition theorem (privacy), and Eq.~\eqref{eq:com} (communication).

Formally, the problem of \dme is defined as given $n$ vectors $X \defeq \{X_1 \ldots X_n\}$ where $X_i \in \reals^d$ is on client $i$, we wish to compute the mean 
\[
\bar{X} = \frac{1}{n}\sum_{i = 1}^{n} X_i
\]
%\end{equation}
at a central server. For gradient descent at each round $t$, $X_i$ is set to $g^t_i$. \dmecap is a fundamental building block for many distributed learning algorithms including distributed PCA/clustering~\cite{Lloyd82}.

While analyzing private \dme we assume that each vector $X_i$ has bounded $\ell_2$ norm, i.e. $\|X_i\| \leq \bound$. The reason to make such an assumption is to be able to define and analyze the privacy guarantees and is often enforced in practice by employing gradient clipping at each client. We note that this assumption appears in previous works on gradient descent and differentially private gradient descent (e.g. \cite{abadi2016deep}). Since our results also hold for all gradients without any statistical assumptions, we get desired convergence results and privacy results for SGD.

\subsection{Communication protocol}
Our proposed communication algorithms are simultaneous and independent, i.e., the clients independently send data to the server at the same time. We allow the use of both private and public randomness. Private randomness refers to random values generated by each client separately, and public randomness refers to a sequence of
random values that are shared among all parties\footnote{Public randomness can be emulated by the server communicating a random seed}.

 Given $n$ vectors $X \defeq \{X_1 \ldots X_n\}$ where $X_i \in \reals^d$ resides on a client $i$. 
In any independent communication protocol, each client
transmits a function of $X_i$ (say $q(X_i)$), and a central server
estimates the mean by some function of $q(X_1),q(X_2),\ldots,
q(X_n)$. Let $\pi$ be any such protocol and let $\cC_i(\pi, X_i)$ be
the expected number of transmitted bits by the $i$-th client during protocol
$\pi$, where throughout the paper, expectation is over the randomness in
protocol $\pi$.

The total number of bits transmitted by all clients with the protocol $\pi$ is 
\[
\cC(\pi, X^n_1) \ed  \sum^n_{i=1} \cC_i(\pi, X_i).
\]
Let the estimated mean be $\hat{\bar{X}}$. For a protocol $\pi$, the MSE of the estimate is
\[
\cE(\pi, X^n_1) = \expect\left[\norm{\hat{\bar{X}} - \bar{X}}^2_2\right].
\]
We note that bounds on $\cE((\pi, X^n_1)$, translates to bounds on gradients estimates in Eq.~\eqref{eq:gvarq} and result in convergence guarantees via Corollary~\ref{cor:qsgd}.
\subsection{Differential privacy}
\label{sec:diff_privacy}

To state the privacy results for DME, we define the notion of data sets and neighbors as follows. 
A dataset is a collection of vectors  $X = \{X_1, \ldots X_n\}$.
The notion of neighboring data sets typically corresponds to those differing only on the information of one user, i.e. $X, X_{\otimes i}$ are neighbors if they differ in one vector.

Note that this notion of neighbors for \dme in the context of distributed gradient descent translates to two data sets 
\[ F = f_1,f_2,\ldots f_n \text{ and } F' = f'_1,f'_2,\ldots f'_n\]
 being neighbors if they differ in one function $f_i$ 
% i.e.,
% \[ F' = \{ f_1 \ldots f_{i-1}, f'_i, f_{i+1} \ldots f_n\}.\]
and corresponds to guaranteeing privacy for individual client's data. 
The bound $\norm{X_i}_2 \leq D$ translates to assuming $\|g^t_i\| \leq D$, ensured via gradient clipping.

\section{Results for distributed mean estimation (DME)}
\label{sec:results}
In this section we describe our algorithms, the associated MSE, and the privacy guarantees in the context of DME. First, we first establish a baseline by stating the results for implementing the Gaussian mechanism by adding Gaussian noise on each client vector.
\subsection{Gaussian protocol}
In the Gaussian mechanism, each client sends vector
\[
Y_i = X_i + Z_i,
\]
where $Z_i$s are i.i.d distributed as $\normal(0, \sigma^2 \II_d)$. The server estimates the mean by 
\[\hat{\bar{X}} = \frac{1}{n} \cdot \sum_{i = 1}^{n} Y_i.\]
We refer to this protocol as $\pi_g$. Since $\sum_{i=1}^n Z_i/ n$ is distributed as $\normal(0, {\sigma^2} \II_d / n)$ the above mechanism is equivalent to applying the Gaussian mechanism on the output with variance $\sigma^2/n$. 
Since changing any of the $X_i$'s changes the norm of $\bar{X}$ by at most $2\bound/n$, the following theorem follows directly from Lemma~\ref{lem:Gaussian_overall}.
%a standard result on Gaussian mechanism \cite{dworkroth} implies that, if 
% \[
% \sigma^2 \geq \frac{8\bound^2}{n \epsilon^2} \log \frac{1.25}{\delta},
% \]
% then the resulting mechanism is $\epsilon, \delta$ differentially private. 

%\na{I made the constants correct in the above. Please see if we wish to leave constants in the theoorem below or leave it $O()$ because our statements are in $O()$}. A simple analysis shows that 
\begin{theorem}
\label{thm:gaussian}
Under the Gaussian mechanism, the mean estimate is unbiased and communication cost is $n\cdot d$ reals. Moreover, for any $\delta$ and $\sigma \geq \frac{2D}{\sqrt{n}} \cdot \sqrt{2 \log {1.25}/{\delta}}$, it is $(\epsilon,\delta)$ differentially private for %$\epsilon$ such that  
\[
\epsilon =\frac{2D}{\sqrt{n}\sigma} \sqrt{2 \log \frac{1.25}{\delta}} \;\;\text{ and }\;\; \cE(\pi_g, X) = \frac{d\sigma^2}{n},
\]
% and the error of the mechanism is 
% \[
% d \cdot \sigma^2.
% \]
% Furthermore
% \[

% \]
% For every $\epsilon, \delta$, $\epsilon, \delta$-Gaussian mechanism has 
% \[
% \cE(\pi_g, X) =  \frac{8d\bound^2}{n^2 \epsilon^2} \log \frac{1.25}{\delta},
% \]t
\end{theorem}

We remark that real numbers can be quantized to $\cO(\log dn/\epsilon\delta)$ bits with insignificant effect to privacy\footnote{Follows by observing that quantizing all values to $1/\text{poly}(n,d,1/\epsilon, \log 1/\delta)$ accuracy ensures minimum loss in privacy. In practice this is often implemented using 32 bits of quantization via float representation.}.
However this is asymptotic and can be prohibitive in practice~\cite{konevcny2016federated}, where we have a small fixed communication budget and $d$ is of the order of millions. A natural way to reduce communication cost is via quantization, where each client quantizes $Y_i$s before transmitting. However how privacy guarantees degrade due to quantization of the Gaussian mechanism is  hard to analyze particularly under aggregation. Instead we propose to use the Binomial mechanism which we describe next.
% To overcome the above problem we propose the Binomial Mechanism in which each client adds an appropriately parameterized Binomial random variable as noise. Note that Binomial is a discrete distribution and therefore finds natural application into fixed bit communication settings and in fact can be seen as an analogue of the Gaussian noise in the discrete case. A useful property of Binomial noise (similar to Gaussian noise) is that it is closed under addition. Further the Binomial distribution has finite support and is therefore easier to analyze than Poisson distribution. \na{This has already been mentioned once before. Do we need to say this again.}
% In the next section we provide a formal description of the Binomial mechanism and provide the formal guarantees regarding the privacy-accuracy-communication trade-offs.  
% We show that the communication cost can be reduced to $nd + \tilde\cO(1)$, which in practice often results in $32x$ reduction.
%  However, it is not clear if the resulting protocol is differentially private. Thus, we first show that the stochastic binary quantization scheme proposed in ~\cite{suresh2016distributed} can be modified slightly to obtain differential privacy. 
\subsection{Stochastic $k$-level quantization + Binomial mechanism}
\label{sec:BINdescription}
We now define the mechanism $\binomprotocol$ based on $k$-bit stochastic quantization $\pi_{sk}$ proposed in \cite{suresh2016distributed} composed with the Binomial mechanism. It will be parameterized by $3$ quantities $k, m, p$.

First, the server sends $X^{\max}$ to all the clients, with the hope that
for all $i,j$, $-X^{\max} \leq X_i(j) \leq X^{\max}$.  The clients then clip each coordinate of their vectors to the range $[-X^{\max}, X^{\max}]$.
% i.e., for all $i,j$, they set
% \[
% X_{i}(j) \leftarrow \max(-X^{\max}, \min(X^{\max}, X_i(j))).
% \]
%its guess for the minimum and maximum values for all the coordinates of the vectors $X_i$.
%We note that for a naive implementation, $X^{\min}$ can be set to $-\bound$ and $X^{\max}$ can be set to $\bound$.
%, where $\bound$ is an upper bound on $\max_{i} \norm{X_i}_2$. 
%Further consider the following definitions.  
For every integer $r$ in the range $[0,k)$,  let $B(r)$represent a bin (one for each $r$), i.e.
\begin{equation}
\label{eqn:binomquantbins}
B(r) \ed -X^{\max} +  \frac{2 r X^{\max}}{k-1}, 
\end{equation}
The algorithm
quantizes each coordinate into one of the bins stochastically and adds scaled Binomial noise. Formally client $i$ computes the following quantities for every $j$
\begin{equation}
\label{eqn:binomquantui}
\binomquantui_i(j) = 
\begin{cases}
    B(r+1)  & \text{w.p. } \frac{X_i(j) - B(r)}{B(r+1) - B(r)} \\
    B(r) & \text{otherwise.}
\end{cases} \quad\quad\quad
Y_i(j) = 
\binomquantui_i(j) + \frac{2X^{\max}}{k-1} \cdot \binomnoise_i(j).
\end{equation}
where $r$ is such that $X_i(j) \in [B(r), B(r+1)]$ and $T_i(j) \sim \Bin(m,p)$. The client sends $Y_i$ to the server. The server now estimates $\bar{X}$ by
%Formally in $\binomprotocol$, for the $i$-th data point and $j$-th coordinate, %
 %let $r$ be defined such that 
% %Let $p_i(j) = \frac{X_i(j) - B_i(r)}{B_i(r+1) - B_i(r)}$. 
% Client $i$ sends $Y_i$ such that each coordinate $j$ of $Y_i$ is defined as:  
% \begin{equation}
% \label{eqn:binomquantyi}
% Y_i(j) = 
% \begin{cases}
% X^{\max}  + \frac{X^{\max}- X^{\min}}{k-1} \cdot \binomnoise_i(j)  &  \text{if } X_i(j) > X^{\max},\\
% X^{\min}  + \frac{X^{\max}- X^{\min}}{k-1} \cdot \binomnoise_i(j)  & \text{if } X_i(j) < X^{\min}, \\
% \binomquantui_i(j) + \frac{X^{\max}- X^{\min}}{k-1} \cdot \binomnoise_i(j) & \text{otherwise,}
% \end{cases}
% \end{equation}
% %then $Y_{i}(j) = B_i({r+1})$ with probability
% %\[
% %\frac{X_i(j) - B_i(r)}{B_i(r+1) - B_i(r)},
% %\]
% %and $Y_i(j) = B_i(r)$ otherwise. 
% where each $T_i(j)$ distributed according to $\Bin(m,p)$. 
% While under the naive implementation we are working under the assumption that $X_i(j) \in [X^{\min}, X^{\max}]$, we define the mechanism for the general setting when the assumption might not be satisfied. This is useful when we consider mechanisms with rotation in future sections. 
\begin{equation}
\label{eqn:binomquantestimator}
\binomquantestimator = \frac{1}{n} \sum^n_{i=1} \left( Y_i - \frac{2X^{\max}mp}{k-1} \right). 
\end{equation}
If $\forall j$, $X_i(j) \in [-X^{\max}, X^{\max}]$, then 
\[\expect\left[Y_i - \frac{2X^{\max}mp}{k-1}\right] = X_i,\]
 and $\binomquantestimator$ will be an unbiased estimate of the mean. 
With no prior information on $X^{\max}$, the natural choice is to set $X^{\max} = D$. With this value of $X^{\max}$ we characterize the MSE, sensitivity, and communication complexity of the Binomial mechanism below. To characterize the sensitivity of $\hat{\bar{X}}_{\binomprotocol}$, we need few definitions. For scalars $D, X^{\max}$, let 
\begin{align}
&\Delta_\infty(X^{\max}, D) \stackrel{\text{def}}{=} k + 1 \nonumber \\
&\Delta_1(X^{\max}, D) \stackrel{\text{def}}{=}  \frac{\sqrt{d}D}{q} + \sqrt{\frac{2\sqrt{d}D \log (2/\delta)}{q}} + \frac{4}{3} \log \frac{2}{\delta} \nonumber \\
&\Delta_2(X^{\max}, D) \stackrel{\text{def}}{=} \frac{D}{q} + \sqrt{\Delta_1 + \sqrt{\frac{2\sqrt{d}D \log (2/\delta)}{q}} }, \label{eq:d1d2d3}
\end{align}
where $q = X^{\max}/(k-1)$. We note that quantities $k,\delta, \epsilon, d$ are omitted from the LHS of the equations for the ease of notation. Combined with Theorem~\ref{thm:discbinomial}, this yields the privacy guarantees for the binomial mechanism.
\begin{theorem}
\label{thm:MSEbinomial}
If $X^{\max} = \bound$, then the mean estimate is unbiased and 
    \[
 \cE\left(\binomprotocol, X^n \right) \leq  \frac{dD^2}{n(k-1)^2} + \frac{d}{n}\cdot\frac{4mp(1-p) D^2}{(k-1)^2},
    \]
Furthermore if
     $mnp(1-p) \geq \max \left(23\log(10d/\delta), 2\Delta_{\infty}(\bound, X^{max}) \right)$,
  then for any $\delta$, $\hat{\bar{X}}_{\binomprotocol}$ is $(\epsilon, 2\delta)$ differentially private where $\epsilon$ is given by Theorem \ref{thm:discbinomial} with sensitivity parameters $\{\Delta_1(X^{max}, D), \Delta_2(X^{max}, D), \Delta_\infty(X^{max}, D)\}$ (Eq.~\eqref{eq:d1d2d3}).  Furthermore, 
\[
\cC(\binomprotocol, X^n) = n \cdot (d \log_2(k+m) + \tilde\cO(1) ).\footnote{$\tilde{\cO}$ is used to denote poly-logarithmic factors.}
\]    
\end{theorem}
We provide the proof in Appendix \ref{sec:dmebinmechproof}. For $\epsilon \leq 1$, we bound the communication cost as follows. 
\begin{corollary}
\label{corr:asymptoticbinom}
There exists an implementation of $\binomprotocol$, which achieves the same privacy and error as the full precision Gaussian mechanism with a total communication complexity of
  \[ n \cdot d \cdot \left(\log_2 \left(\sqrt{d} + \frac{d}{n \epsilon^2} \right) +  \cO( \log \log \left(\frac{nd}{\epsilon \delta} \right))\right)   \text{  bits.} \] 
\end{corollary}
Therefore our results provide precise non-asymptotic and asymptotic guarantees on the total communication with respect to k. 
The communication cost of the above algorithm is $\Omega(\log d)$ bits per coordinate per client, which can be prohibitive. In the next section we show that these bounds can be improved via rotation.

\subsection{Error reduction via randomized rotation}
\label{sec:rotation}
As seen in Corollary~\ref{corr:asymptoticbinom}, if  $\binomprotocol$ has error and privacy same as that of the Gaussian mechanism, it has high communication cost.
The proof reveals that this is due to the error being proportional to $O({d}(X^{\max})^2/{n})$. Therefore MSE reduces when $X^{\max}$ is small, e.g., when $X_i$ is uniform on the unit sphere, $X^{\max}$ is $\cO\left(\sqrt{(\log d)/d}\right)$ (whp) \cite{dasgupta2003elementary}.
\cite{suresh2016distributed} showed that the same effect can be observed by randomly rotating the vectors before quantization. Here we show that random rotation reduces the leading term as well as  improves the privacy guarantee. 

Using public randomness, all clients and the central server generate a
random orthogonal matrix $R \in \RR^{d\times d}$ according to some known distribution.
Given a protocol $\pi$ for \dme which takes inputs $X_1 \ldots X_n$, we define $\Rot(\pi, R)$ as the protocol where each client $i$ first computes, 
\[
X'_i = R X_i,
\]
and runs the protocol on $X'_1,X'_2,\ldots X'_n$. The server then obtains the mean estimate $\hat{\bar{X'}}$ in the rotated space using the protocol $\pi$ and then multiplies by $R^{-1}$ to obtain the coordinates in the original basis, i.e.,
\[
\hat{\bar{X}} = R^{-1} \hat{\bar{X'}}.
\]
Due to the fact that $d$ can be huge in practice, we need orthogonal matrices that permits fast matrix-vector products. Naive matrices that
support fast multiplication such as block-diagonal matrices often
result in high values of $\|X'_i\|_{\infty}^2$.  Similar to \cite{suresh2016distributed}, we propose to use a special type of orthogonal matrix $R = \frac{1}{\sqrt{d}} HA$, where $A$ is a random diagonal matrix
with i.i.d. Rademacher entries ($\pm 1$ with probability $0.5$) and $H$ is a Walsh-Hadamard matrix \cite{horadam2012hadamard}. 
The Walsh-Hadamard matrix of dimension $2^m$ for $m \in \mathcal{N}$ is given by the recursive formula, 
\begin{align*}
H(2^1) = \begin{bmatrix}
1 & 1\\
1 & -1\end{bmatrix}, 
H(2^m) = \begin{bmatrix}
H(2^{m-1}) &  H(2^{m-1})\\
H(2^{m-1})  & -H(2^{m-1})\end{bmatrix}.
\end{align*}
Applying both rotation and its
inverse takes $\cO(d \log d)$ time and $\cO(1)$ space
(with an in-place algorithm). 

% In this section we describe a general scheme for improving the privacy vs error trade-off for distributed mean estimation via rotation. We assume that the client and the server have access to a public stream of randomness. In practice this can be implemented by agreeing on a seed for a random number generator before the initiation of a protocol. 

% Given a protocol \textbf{P} for mean estimation for a database $X = \{X_1 \ldots X_n\}$ the modified protocol is as follows. Each client samples a random rotation matrix $R$ from the public stream of randomness. Note that every client samples the same rotation matrix $R$ and defines a vector $Y_i = R X_i$. Further it applies the protocol $P$ on the vector $Y_i$ and sends it to the server. 

% The server on receipt of the input vectors applies the decoding scheme specified by the protocol $P$ and produces the estimator $\tilde{\bar{Y}}$. The server then samples the same rotation matrix $R$ from the public stream of randomness and outputs the estimator $\tilde{\bar{X}} = R^{-1} \tilde{\bar{Y}}$. For a general protocol \textbf{P} we refer to the rotated protocol as \textbf{R-P}. In particular let protocols $\Rot(\binaryprotocol, HA)$ and \textbf{R-BIN($k,n,p$)} be the rotated versions of the protocols $\binaryprotocol$ and \textbf{BIN($k,n,p$)}. We have the following theorems for $\Rot(\binaryprotocol, HA)$ and \textbf{R-BIN($k,n,p$)}

The next theorem provides the MSE guarantees for $\Rot(\binomprotocol, HA)$.
\begin{theorem}[Appendix~\ref{sec:rotateddmebinmechproof}]
\label{thm:mainthmbinomialrotated}
For any $\delta$, let $X^{\max} = 2 \bound\sqrt{\frac{\log(2nd/\delta)}{d}}$,
%     Then for every $\epsilon, \delta,k$, and $m, n, p$ such that 
%    \[mnp(1-p) \geq \max \left\{2(k+1), 23 \log(20d/\delta) \right\}\]
%     the protocol $Rot(\binomprotocol)$ is $(\epsilon, \delta + \delta')$ differentially private for 
%     \begin{multline*}
% \epsilon = \frac{(k-1)\sqrt{2d\log (2.5/\delta)}}{2\sqrt{mnp(1-p) \log(\frac{2nd}{\delta'})}}
%  + \frac{2 (k+1) \log (2.5/\delta)}{3mnp(1-p)} ,\nonumber\\
%  + \frac{\sqrt{4\gamma\log (2.5/\delta)}}{\sqrt{mnp(1-p)}}+  \frac{5 \gamma (p^2+(1-p)^2)\log (20d/\delta)}{3mnp(1-p)}
%     \end{multline*}
%     where $
%     \gamma \defeq 3d(k-1) + \frac{10}{3} \log \frac{2}{\delta}$. 
 then 
    \[ \cE(\Rot(\pi_{sk}(\Bin(m,p))), HA) \leq  \frac{2 \log \frac{2nd}{\delta} \cdot D^2}{n(k-1)^2} + \frac{8 \log \frac{2nd}{\delta} \cdot mp(1-p) D^2}{n(k-1)^2} + 4\bound^2\delta^2\]
     and the bias is $\leq 2\bound\delta$. Further if
     $mnp(1-p) \geq \max \left(23\log(10d/\delta), 2\Delta_{\infty}(\bound, X^{max}) \right),$ then
      $\hat{\bar{X}}(\Rot(\binomprotocol))$ is $(\epsilon, 3\delta)$ differentially private where $\epsilon$ is given by Theorem \ref{thm:discbinomial} with sensitivity parameters $\{\Delta_1(X^{max}, D), \Delta_2(X^{max}, D), \Delta_\infty(X^{max}, D)\}$ (Eq.~\eqref{eq:d1d2d3}).  Furthermore, 
\[
\cC(\Rot(\binomprotocol), X^n) = n \cdot (d \log_2(k+m) + \tilde\cO(1) ).
\]    
%     Furthermore, with probability $1-\delta'$, 
%      Furthermore, there exists an implementation of the algorithm such that 
% \[
% \cC(\binomprotocol, X^n) = n \cdot (d \left(\log(k+m) + 1\right) + \tilde\cO(1) ).
% \]    
\end{theorem}
The following corollary bounds the communication cost for  $\Rot(\binomprotocol, HA)$ when $\epsilon \leq 1$. 
\begin{corollary}
\label{corr:asymptoticrotbinom}
There exists an implementation of $\Rot(\binomprotocol, HA)$, that achieves the same error and privacy of the 
%   For the problem of distributed mean estimation, there exists a mechanism (namely $\Rot(\binomprotocol)$) which can achieve the same privacy vs error trade off as the
    full precision Gaussian mechanism with a total communication complexity:
  \[ n \cdot d \left(\log_2 \left(1 + \frac{d}{n \epsilon^2} \right) + \cO \left( \log \log \frac{dn}{\epsilon\delta} \right) \right) \text{  bits.} \] 
\end{corollary}
Hence if $d = \cO(n\epsilon^2)$, then $\Rot(\binomprotocol, HA)$ has the same privacy and utilities as the Gaussian mechanism, but with just $\cO(nd \log \log (nd/\delta\epsilon))$ communication cost. 
% As compared to Corollary \ref{corr:asymptoticbinom} we see that in the case when $\epsilon$ is a constant and $n = \Omega(d)$ we get that the rotated binomial protocol has an overall communication complexity of $\cO(nd)$ to achieve the privacy vs error guarantees of the Gaussian mechanism as compared to the $\cO(nd \log(d))$ guaranteed by the unrotated binomial protocol.  

\section{Discussion} 
%We studied the problem of communication-efficient and differentially-private distributed SGD.

%We first analyzed Binomial mechanism and showed that it achieves error similar to that of the Gaussian mechanism, with fewer bits necessary for representation. We then reduced distributed SGD to the problem of distributed mean estimation. We showed that Binomial mechanism coupled with random rotation pre-processing has similar error to that of Gaussian mechanism with $\cO^*(1)$ bits of communication per client per round when $n \approx d$. 

We trained a three-layer model (60 hidden nodes each with ReLU activation) on the infinite MNIST dataset~\cite{mnistinfinity} with 25M data points and 25M clients. At each step 10,000 clients send their data to the server. This setting is close to real-world settings of federated learning where there are hundreds of millions of users. The results are in Figure~\ref{fig:experiment}. Note that the models achieve different levels of accuracy depending on communication cost and privacy parameter $\epsilon$. We note that we trained the model with exactly one epoch, so each sample was used at most once in training. In this setting, the per batch $\epsilon$ and the overall $\epsilon$ are the same.

There are several interesting future directions. On the theoretical side, it is not clear if our analysis of Binomial mechanism is tight. Furthermore, it is interesting to have better privacy accounting for Binomial mechanism via a moments accountant. On the practical side, we plan to explore the effects of neural network topology, over-parametrization, and optimization algorithms on the accuracy of the privately learned models.

\section{Acknowledgements}
The authors would like to thank Keith Bonawitz, Vitaly Feldman, Jakub Konečný, Ben Kreuter, Ilya Mironov, and Kunal Talwar for their valuable suggestions and inputs. 

\bibliographystyle{plain}
\bibliography{refs2}

\begin{thebibliography}{10}

\bibitem{abadi2016tensorflow}
Mart{\'\i}n Abadi, Ashish Agarwal, Paul Barham, Eugene Brevdo, Zhifeng Chen,
  Craig Citro, Greg~S Corrado, Andy Davis, Jeffrey Dean, Matthieu Devin, et~al.
\newblock Tensorflow: Large-scale machine learning on heterogeneous distributed
  systems.
\newblock {\em arXiv preprint arXiv:1603.04467}, 2016.

\bibitem{abadi2016deep}
Mart{\'\i}n Abadi, Andy Chu, Ian Goodfellow, H~Brendan McMahan, Ilya Mironov,
  Kunal Talwar, and Li~Zhang.
\newblock Deep learning with differential privacy.
\newblock In {\em Proceedings of the 2016 ACM SIGSAC Conference on Computer and
  Communications Security}, pages 308--318. ACM, 2016.

\bibitem{AilonL09}
Nir Ailon and Bernard Chazelle.
\newblock Approximate nearest neighbors and the fast {J}ohnson-{L}indenstrauss
  transform.
\newblock In {\em STOC}, 2006.

\bibitem{alistarh2017communication}
Dan Alistarh, Demjan Grubic, Jerry Liu, Ryota Tomioka, and Milan Vojnovic.
\newblock Communication-efficient stochastic gradient descent, with
  applications to neural networks.
\newblock 2017.

\bibitem{AlistarhLTV16}
Dan Alistarh, Jerry Li, Ryota Tomioka, and Milan Vojnovic.
\newblock {QSGD}: Randomized quantization for communication-optimal stochastic
  gradient descent.
\newblock {\em arXiv:1610.02132}, 2016.

\bibitem{bassily2014private}
Raef Bassily, Adam Smith, and Abhradeep Thakurta.
\newblock Private empirical risk minimization: Efficient algorithms and tight
  error bounds.
\newblock In {\em Foundations of Computer Science (FOCS), 2014 IEEE 55th Annual
  Symposium on}, pages 464--473. IEEE, 2014.

\bibitem{bonawitz2016practical}
Keith Bonawitz, Vladimir Ivanov, Ben Kreuter, Antonio Marcedone, H~Brendan
  McMahan, Sarvar Patel, Daniel Ramage, Aaron Segal, and Karn Seth.
\newblock Practical secure aggregation for privacy-preserving machine learning.
\newblock pages 1175--1191, 2017.

\bibitem{mnistinfinity}
Leon Bottou.
\newblock The infinite mnist dataset, 2007.

\bibitem{coates2013deep}
Adam Coates, Brody Huval, Tao Wang, David Wu, Bryan Catanzaro, and Ng~Andrew.
\newblock Deep learning with cots hpc systems.
\newblock In {\em International Conference on Machine Learning}, pages
  1337--1345, 2013.

\bibitem{dasgupta2003elementary}
Sanjoy Dasgupta and Anupam Gupta.
\newblock An elementary proof of a theorem of johnson and lindenstrauss.
\newblock {\em Random Structures \& Algorithms}, 22(1):60--65, 2003.

\bibitem{dean2012large}
Jeffrey Dean, Greg Corrado, Rajat Monga, Kai Chen, Matthieu Devin, Mark Mao,
  Andrew Senior, Paul Tucker, Ke~Yang, Quoc~V Le, et~al.
\newblock Large scale distributed deep networks.
\newblock In {\em Advances in neural information processing systems}, pages
  1223--1231, 2012.

\bibitem{dwork2006our}
Cynthia Dwork, Krishnaram Kenthapadi, Frank McSherry, Ilya Mironov, and Moni
  Naor.
\newblock Our data, ourselves: Privacy via distributed noise generation.
\newblock In {\em Eurocrypt}, volume 4004, pages 486--503. Springer, 2006.

\bibitem{dwork2006calibrating}
Cynthia Dwork, Frank McSherry, Kobbi Nissim, and Adam Smith.
\newblock Calibrating noise to sensitivity in private data analysis.
\newblock In {\em TCC}, volume 3876, pages 265--284. Springer, 2006.

\bibitem{dworkroth}
Cynthia Dwork and Aaron Roth.
\newblock The algorithmic foundations of differential privacy.
\newblock {\em Found. Trends Theor. Comput. Sci.}, 9(3\&\#8211;4):211--407,
  August 2014.

\bibitem{elias1975}
Peter Elias.
\newblock Universal codeword sets and representations of the integers.
\newblock {\em IEEE transactions on information theory}, 21(2):194--203, 1975.

\bibitem{ghadimi2013stochastic}
Saeed Ghadimi and Guanghui Lan.
\newblock Stochastic first-and zeroth-order methods for nonconvex stochastic
  programming.
\newblock {\em SIAM Journal on Optimization}, 23(4):2341--2368, 2013.

\bibitem{gupta2015deep}
Suyog Gupta, Ankur Agrawal, Kailash Gopalakrishnan, and Pritish Narayanan.
\newblock Deep learning with limited numerical precision.
\newblock In {\em Proceedings of the 32nd International Conference on Machine
  Learning (ICML-15)}, pages 1737--1746, 2015.

\bibitem{horadam2012hadamard}
Kathy~J Horadam.
\newblock {\em Hadamard matrices and their applications}.
\newblock Princeton university press, 2012.

\bibitem{kairouz2017composition}
Peter Kairouz, Sewoong Oh, and Pramod Viswanath.
\newblock The composition theorem for differential privacy.
\newblock {\em IEEE Transactions on Information Theory}, 63(6):4037--4049,
  2017.

\bibitem{konevcny2016federated}
Jakub Kone{\v{c}}n{\`y}, H~Brendan McMahan, Felix~X Yu, Peter Richt{\'a}rik,
  Ananda~Theertha Suresh, and Dave Bacon.
\newblock Federated learning: Strategies for improving communication
  efficiency.
\newblock {\em arXiv preprint arXiv:1610.05492}, 2016.

\bibitem{konevcny2016randomized}
Jakub Kone{\v{c}}n{\`y} and Peter Richt{\'a}rik.
\newblock Randomized distributed mean estimation: Accuracy vs communication.
\newblock {\em arXiv preprint arXiv:1611.07555}, 2016.

\bibitem{li2014scaling}
Mu~Li, David~G Andersen, Jun~Woo Park, Alexander~J Smola, Amr Ahmed, Vanja
  Josifovski, James Long, Eugene~J Shekita, and Bor-Yiing Su.
\newblock Scaling distributed machine learning with the parameter server.
\newblock In {\em OSDI}, volume~1, page~3, 2014.

\bibitem{li2014communication}
Mu~Li, David~G Andersen, Alexander~J Smola, and Kai Yu.
\newblock Communication efficient distributed machine learning with the
  parameter server.
\newblock In {\em Advances in Neural Information Processing Systems}, pages
  19--27, 2014.

\bibitem{lin2018deep}
Yujun Lin, Song Han, Huizi Mao, Yu~Wang, and Bill Dally.
\newblock Deep gradient compression: Reducing the communication bandwidth for
  distributed training.
\newblock {\em International Conference on Learning Representations}, 2018.

\bibitem{Lloyd82}
Stuart Lloyd.
\newblock Least squares quantization in {PCM}.
\newblock {\em IEEE Transactions on Information Theory}, 28(2):129--137, 1982.

\bibitem{McdonaldGM10}
Ryan McDonald, Keith Hall, and Gideon Mann.
\newblock Distributed training strategies for the structured perceptron.
\newblock In {\em HLT}, 2010.

\bibitem{fedlearning}
H.~Brendan McMahan, Eider Moore, Daniel Ramage, Seth Hampson, and Blaise~Aguera
  y~Arcas.
\newblock Communication-efficient learning of deep networks from decentralized
  data.
\newblock In {\em Proceedings of the 20th International Conference on
  Artificial Intelligence and Statistics (AISTATS)}, 2016.

\bibitem{mcmahan2016federated}
H.~Brendan McMahan, Eider Moore, Daniel Ramage, and Blaise~Aguera y~Arcas.
\newblock Federated learning of deep networks using model averaging.
\newblock {\em arXiv:1602.05629}, 2016.

\bibitem{PoveyZK14}
Daniel Povey, Xiaohui Zhang, and Sanjeev Khudanpur.
\newblock Parallel training of deep neural networks with natural gradient and
  parameter averaging.
\newblock {\em arXiv preprint}, 2014.

\bibitem{rakhlin2012making}
Alexander Rakhlin, Ohad Shamir, Karthik Sridharan, et~al.
\newblock Making gradient descent optimal for strongly convex stochastic
  optimization.
\newblock In {\em ICML}. Citeseer, 2012.

\bibitem{recht2011hogwild}
Benjamin Recht, Christopher Re, Stephen Wright, and Feng Niu.
\newblock Hogwild: A lock-free approach to parallelizing stochastic gradient
  descent.
\newblock In {\em Advances in neural information processing systems}, pages
  693--701, 2011.

\bibitem{sarwate2013signal}
Anand~D Sarwate and Kamalika Chaudhuri.
\newblock Signal processing and machine learning with differential privacy:
  Algorithms and challenges for continuous data.
\newblock {\em IEEE signal processing magazine}, 30(5):86--94, 2013.

\bibitem{seide20141}
Frank Seide, Hao Fu, Jasha Droppo, Gang Li, and Dong Yu.
\newblock 1-bit stochastic gradient descent and its application to
  data-parallel distributed training of speech dnns.
\newblock In {\em Fifteenth Annual Conference of the International Speech
  Communication Association}, 2014.

\bibitem{suresh2016distributed}
Ananda~Theertha Suresh, X~Yu Felix, Sanjiv Kumar, and H~Brendan McMahan.
\newblock Distributed mean estimation with limited communication.
\newblock In {\em International Conference on Machine Learning}, pages
  3329--3337, 2017.

\bibitem{wen2017terngrad}
Wei Wen, Cong Xu, Feng Yan, Chunpeng Wu, Yandan Wang, Yiran Chen, and Hai Li.
\newblock Terngrad: Ternary gradients to reduce communication in distributed
  deep learning.
\newblock {\em arXiv preprint arXiv:1705.07878}, 2017.

\bibitem{wu2017bolt}
Xi~Wu, Fengan Li, Arun Kumar, Kamalika Chaudhuri, Somesh Jha, and Jeffrey
  Naughton.
\newblock Bolt-on differential privacy for scalable stochastic gradient
  descent-based analytics.
\newblock In {\em Proceedings of the 2017 ACM International Conference on
  Management of Data}, pages 1307--1322. ACM, 2017.

\bibitem{tsuzuku2018variancebased}
Takuya~Akiba Yusuke~Tsuzuku, Hiroto~Imachi.
\newblock Variance-based gradient compression for efficient distributed deep
  learning, 2018.

\bibitem{zaremba2014recurrent}
Wojciech Zaremba, Ilya Sutskever, and Oriol Vinyals.
\newblock Recurrent neural network regularization.
\newblock {\em arXiv preprint arXiv:1409.2329}, 2014.

\end{thebibliography}

%!TEX root = sample.tex
\appendix
\onecolumn

\section{Proof of biased SGD}
\label{appsec:biasedsgd}
% \begin{lemma}[ \cite{ghadimi2013stochastic} restated]
% \label{lem:sgd}
% Suppose $F$ is $L$-smooth i.e. there exists $L$ such that 
% \[ \forall \;\;x,y\;\;\|\nabla F(x) - \nabla F(y)\| \leq L\|x - y\|.
% \] and suppose that $\forall x\;\;\; \|\nabla F(x)\| \leq \bound$. 
% Let $w^0$ be given such that $F(w^0) - \min_w F(w) \leq D_F$. Suppose the above algorithm is run for $T$ rounds and $\gamma$ is selected as
% \[ \gamma \defeq \min\left\{\frac{1}{L}, \frac{\sqrt{2D_F}}{\sigma \sqrt{LT}}\right\},\]
% then
% \[
% \expect_{t \sim \mathrm{Uniform}([T])}[ \|\nabla F(w^t)\|^2_2] \leq \frac{2D_FL}{T} + \frac{2\sqrt{2}\sigma\sqrt{LD_F}}{\sqrt{T}} + \bound B,
% \]
% where $\sigma$ is given by,
% \begin{equation*}
% \sigma^2 =  \max_{1\leq t \leq T} \expect[\|g^t(w^t) - \nabla F(w^t)\|^2_2],
%  \end{equation*}
%  and B is given by
%  \[ B =  \max_{1\leq t \leq T} \|\expect[g^t(w^t) - \nabla F(w^t)]\|_2\]
% \end{lemma}
%\begin{proof}
The proof is similar to the SGD proof of  \cite{ghadimi2013stochastic}, however we account for bias in gradient estimates.
Define the random variable $\delta_t \defeq \tilde{g}_t(w_t) - \nabla F(w_{t-1})$. By the definitions of $L$ and $\gamma$, 
\begin{align*}
  F(w_{t+1}) - F(w_{t}) &\leq \nabla F(w_{t})^T(w_{t+1} - w_{t}) + \frac{L}{2}\|w_{t+1} - w_{t}\|^2 \\
  &\leq -\nabla F(w_{t})^T(\gamma \tilde{g}_t(w_t)) + \gamma^2\frac{L}{2}\|\tilde{g}_t(w_t)\|^2 \\
  &\leq -\gamma(1 - \frac{\gamma L}{2})\|\nabla F(w_{t})\|^2 + \gamma(1 - \gamma L)\|\nabla F(w_{t})\|\| \delta_t\| + \gamma^2\frac{L}{2}\|\delta_t\|^2,
\end{align*}
where the last inequality uses the fact that $\gamma L \leq 1$. Rearranging the above inequality and summing over all $t$ we get that 
\begin{align*}
& \expect_{t \in \text{Uniform}(T)}[\|\nabla F(w_{t})\|^2]\\
    &\leq \frac{1}{T\gamma(2 - \gamma L)}\left(2(F(w_0) - F(w^*)) + T\gamma^2L\expect\|\delta_t\|^2\right) + \frac{2\gamma(1 - \gamma L)}{\gamma(2 - \gamma L)}\left(\frac{1}{T}\sum_{t=1}^{T} \|\nabla F(w_t)\|\|\expect[\delta_t]\|\right) \\
  &\leq \frac{1}{T\gamma(2 - \gamma L)}\left(2D_F + T\gamma^2L\sigma^2\right) + \frac{2\gamma(1 - \gamma L)}{\gamma(2 - \gamma L)}\bound B \\
  &\leq \frac{2D_F}{T} \max\left\{L, \frac{\sigma \sqrt{LT}}{{\sqrt{2D_F}}}\right\} + \frac{\sigma \sqrt{2LD_F}}{\sqrt{T}}  + \bound B\\
  &\leq \frac{2D_FL}{T} + \frac{2\sqrt{2LD_F}\sigma}{\sqrt{T}} + \bound B.
\end{align*}  
%\end{proof}
\section{Binomial Mechanism - Proof of Theorem \ref{thm:discbinomial}}
\label{app_sec:binomial}
To remind the reader, the binomial mechanism for releasing discrete valued queries on a database is defined as follows. Given a set of databases $\D$ and an integer valued query $f: \D \rightarrow \integers^{d}$, the binomial mechanism samples a vector $Z \in \integers^d$ such that all its coordinates are distributed as the binomial distribution with parameters $N,p$, i.e.
\[Z(j) \sim \Bin(N,p)\]
The Binomial mechanism releases the vector $s(Z - Np) + f(D)$ as the output to the query. For the analysis the reader is referred to the definition of $\ell_q$ norm sensitivity $\Delta_q$ for any $q > 0$ defined in \eqref{eqn:deltaqdefn}. The $q$ of interest to us for the Binomial mechanism will be $q = \{1,2,\infty\}$. Since our requirement from the Binomial mechanism will be  symmetric w.r.t. $p$ and $1-p$, throughout this proof, we assume that $p \leq 1/2$. 

To prove Theorem \ref{thm:discbinomial}, we need few auxiliary lemmas. We first state two inequalities which we use through-out the proof.
  \begin{lemma}[Bernstein's inequality]
  \label{thm:bernstein}
    Let $X_1, X_2 \ldots X_n$ be independent random variables such that $E[X_i] = 0$ and $|X_i| \leq M$ w.p. 1. Let $\sigma_i^2 \defeq \EE[X_i^2]$. Then for any $\delta \geq 0$,
    \[ 
    \text{Pr}  \left(   \sum X_i \geq \sqrt{2\sum \sigma_i^2 \log\frac{1}{\delta}} + \frac{2}{3} \cdot M \log\frac{1}{\delta}
    \right) \leq \delta.
        \]
  \end{lemma}
    \begin{lemma}[Efron-Stein inequality]
    \label{lem:efron}
    Let $f$ be a symmetric function of $n$ independent random variables $X_1,X_2,\ldots X_n$. Let $X'_1$ be an i.i.d. copy of $X_1$, then 
    \[
    \text{Var}(f) \leq \frac{n}{2} \cdot \EE \left[ (f(X_1,X_2,\ldots X_n) - f(X'_1,X_2,\ldots X_n))^2\right].
    \]
    \end{lemma}
We use the above two results in the next two lemmas. 
\begin{lemma}
\label{lemma:probboundbinom}
Let $\binomnoise \sim \Bin(N,p)$, $i \in [0,N]$, $t \in \integers$, $i - t \in [0,N]$.
 Then
  \[\frac{\Pr(\binomnoise = i - t)}{\Pr(\binomnoise = i)} \leq 
    \exp \left(t \cdot \log \frac{(i+1)(1-p)}{(N-i+1)p}\right)
    \] 
\end{lemma}
\begin{proof}
  \begin{align}
  \label{eqn:probratiobinom}
    \frac{\Pr(\binomnoise = i - t)}{\Pr(\binomnoise = i)} &\defeq \frac{\binom{N}{i-t}}{\binom{N}{i}} \frac{p^{i-t}(1-p)^{N - i + t}}{p^{i}(1-p)^{N - i}} \nonumber \\
        &= \frac{i! (N - i)!}{(i-t)! (N - i + t)!} \left(\frac{1-p}{p}\right)^t\nonumber \\
          &\leq \left(\frac{(i+1)(1-p)}{(N - i+1)p}\right)^t \nonumber ,
            \end{align}
 where the inequality follows from considering the two cases when $t$ can be positive or negative. 
\end{proof}
% Let $U$ be a dataset over $Z^d$. Let two datasets $U$ and $U'$ be neighbors if
% \[
% \norm{U-U'}_2 \leq \Delta_{\max}.
% \]
% Consider the binomial mechanism which releases $U + T$, where $T$ is a $d$-dimensional vector where each entry is independently sampled from $\Bin(N,1/2)$. 

% The proof of Theorem~\ref{thm:discbinomial} depends on Bernstein's inequality, % and McDiarmid's inequality, 
% which we state for completeness.

%     \begin{lemma}[Mcdiarmid's inequality]
%   \label{thm:bernstein}
%     Let $X_1, X_2 \ldots X_n$ be independent random variables such that $E[X_i] = 0$. Let $f$ be a function of $X_1,\ldots X_n$ such that changing any of the $X_i$s changes $f$ by at most $c_i$, then for any $\delta \geq 0$, the event
%     \[ \sum X_i \geq \sqrt{\sum c_i^2 \log(1/\delta)/2} \] 
%     happens with probability at most $\delta$.
%   \end{lemma}  

\begin{lemma}
\label{lem:new_log_bound}
Let $t_1,t_2,\ldots t_d$ be $d$ real numbers. Let $v_i\sim \Bin(N,p)$ independently such that $Np(1-p) \geq 39$. Let 
$A$ be the event that $\norm{v_i - Np}_\infty\leq \beta$ for some $\beta$, such that $\beta \leq N\min(p,1-p)/3$. Then for any $\delta$, with probability $\geq 1-\delta$ conditioned on $A$, 
\begin{align*}
& \sum^d_{i=1}t_i \left( \cdot \log \frac{(v_i+1)(1-p)}{(N-v_i+1)p} 
  - \frac{v_i + 1}{Np}  + \frac{N-v_i +1}{N(1-p)}
\right)  \\
& \leq \frac{2\norm{t}_1(p^2+(1-p)^2)}{3 Np(1-p)(\Pr(A))} +
\frac{\norm{t}_2 c_p}{Np(1-p)\sqrt{\Pr(A)}}
\cdot \sqrt{ \log \frac{1}{\delta}} + 
\frac{4 \norm{t}_\infty(\beta+1)^2(p^2 + (1-p)^2)}{9N^2p^2(1-p)^2} \log \frac{1}{\delta},
\end{align*}
where $c_p$ is given by
\begin{equation}
\label{eq:cp}
c_p \triangleq  \sqrt{2} (3p^3 + 3(1-p)^3 + 2p^2 + 2(1-p)^2).
\end{equation}
\end{lemma}
\begin{proof}
Since $\beta \leq N\min(p,1-p)/3$ and for any $z \geq -1/3$, $|\log (1+z) - z| \leq 1.95z^2/3$,
\begin{align*}
\left \lvert \log \frac{(v_i+1)(1-p)}{(N-v_i+1)p}  - \frac{v_i + 1}{Np}  + \frac{N-v_i +1}{N(1-p)} \right \rvert
& \leq \frac{1.95}{3} \left \lvert \frac{v_i +1 - Np}{Np} \right \vert^2 +
\frac{1.95}{3}\left \lvert \frac{N - v_i +1 -N - Np}{N(1-p)}  \right \rvert^2.
\end{align*}
Hence we can bound the expectation as
\begin{align*}
 & \EE \left[\log \frac{(v_i+1)(1-p)}{(N-v_i+1)p} - \frac{v_i + 1}{Np}  + \frac{N-v_i +1}{N(1-p)} \bigg| A \right] \\
& \leq \EE \left[
 \frac{1.95}{3} \left \lvert \frac{v_i +1 - Np}{Np} \right \vert^2 +
\frac{1.95}{3}\left \lvert \frac{N - v_i +1 -N - Np}{N(1-p)}  \right \rvert^2
\bigg| A \right] \\
& \stackrel{(a)}{\leq}\frac{1}{\Pr(A)} \cdot \EE \left[
 \frac{1.95}{3} \left \lvert \frac{v_i +1 - Np}{Np} \right \vert^2 +
\frac{1.95}{3}\left \lvert \frac{N - v_i +1 -N - Np}{N(1-p)}  \right \rvert^2
\right] \\
& \stackrel{(b)}{\leq} \frac{1}{\Pr(A)} \cdot \frac{2(p^2 + (1-p)^2)}{3Np(1-p)},
\end{align*}
Where $(a)$ uses the fact that for any positive random variable $X$ and any event $A$, $\EE[X]
\geq \Pr(A) \EE[X|A]$. $(b)$ uses the fact that $Np(1-p) \geq 39$. 
% is same and hence 
% \[
% \EE \left[ \sum^d_{i=1}t_i \cdot \log \frac{(v_i+1)(1-p)}{(N-v_i+1)p} - \frac{v_i + 1}{Np}  + \frac{N-v_i +1}{N(1-p)}  | A \right] = 0.
% \]
% Let $A$ be the event such that,
% \[
% \norm{v_i - Np}_\infty \leq \sqrt{Np(1-p) \log \frac{2d}{\delta}},
% \]
Note that the function we are considering is a sum of functions of $d$ independent binomial random variables and hence we can apply Bernstein' inequality. To this end, we bound $\sigma^2_i$ and $M$. 
Since $\norm{v_i - Np}_\infty$ is bounded, 
\begin{align*}
\left \lvert \log \frac{(v_i+1)(1-p)}{(N-v_i+1)p}  - \frac{v_i + 1}{Np}  + \frac{N-v_i +1}{N(1-p)} \right \rvert
& \leq \frac{2}{3} \left \lvert \frac{v_i +1 - Np}{Np} \right \vert^2 +
\frac{2}{3}\left \lvert \frac{N - v_i +1 -N - Np}{N(1-p)}  \right \rvert^2 \\
& \leq \frac{2}{3}\frac{(\beta + 1)^2(p^2 + (1-p))^2}{N^2p^2(1-p)^2},
\end{align*}
where the first inequality follows from the fact that $\beta \leq N\min(p,1-p)/3$ and for any $z \geq -1/3$, $|\log (1+z) - z| \leq 2z^2/3$. Hence we can set $M =  \frac{2}{3}\frac{(\beta + 1)^2(p^2 + (1-p))^2}{N^2p^2(1-p)^2}$. We now bound the variance:
\begin{align*}
\text{Var} \left(\sum^d_{i=1}t_i \cdot \log \frac{(v_i+1)(1-p)}{(N-v_i+1)p} -   \frac{v_i + 1}{Np}  + \frac{N-v_i +1}{N(1-p)} \bigg| A  \right).
\end{align*}
We now bound $\sigma^2_i$. Observe that the term corresponding to $i$, 
 is a function of $n$ independent Bernoulli $p$ random variables $X_i(j)$, for $1 \leq j \leq d$. We bound the expected square change in the function for any of these variables $X_i(j)$ and then use Efron-Stein inequality. Let $\EE_A$ denote the expectation conditioned on the event $A$. Without loss of generality we first consider the contribution of the term $X_i(j)$.  
 
 Let $w = \sum^n_{j' \neq j} X_i(j')$, then
\begin{align*}
% & \EE_A \left[
% \sum^d_{i=1}t_i \cdot \log \frac{(v_i+1)(1-p)}{(N-v_i+1)p} -   \frac{v_i + 1}{Np}  + \frac{N-v_i +1}{N(1-p)} - \sum^d_{i=1}t_i \cdot \log \frac{(v'_i+1)(1-p)}{(N-v'_i+1)p} -   \frac{v'_i + 1}{Np}  + \frac{N-v'_i +1}{N(1-p)}
% \right]^2 \\
% & \stackrel{(a)}{=}
& \EE_A \left[
t_i \cdot \log \frac{(v_i+1)(1-p)}{(N-v_i+1)p} -   \frac{v_i+ 1}{Np}  + \frac{N-v_i +1}{N(1-p)} - t_i \cdot \log \frac{(v'_i+1)(1-p)}{(N-v'_i+1)p} -   \frac{v'_i + 1}{Np}  + \frac{N-v'_i +1}{N(1-p)}
\right]^2 \\
& = t^2_i\EE_A \left[
\cdot \log \frac{(w+X_i(j)+1)(1-p)}{(N-w-X_i(j)+1)p} -   \frac{w+X_i(j) + 1}{Np}  + \frac{N-w-X_i(j) +1}{N(1-p)} \right]  \\
& - t^2_i\EE_A \left[ \cdot \log \frac{(w+X'_i(j)+1)(1-p)}{(N-w-X'_i(j)+1)p} -   \frac{w+X'_i(j) + 1}{Np}  + \frac{N-w-X'_i(j) +1}{N(1-p)}
\right]^2 \\
& \stackrel{(a)}{=}2 t^2_i p(1-p) \EE_A \left[  \log \left( 1+\frac{1}{w+1}\right) +  \log \left( 1+\frac{1}{N-w}\right)  - \frac{1}{Np} - \frac{1}{N(1-p)} 
\right]^2 \\
& \stackrel{(b)}{\leq} 2 t^2_i p(1-p) \EE \left[  \log \left( 1+\frac{1}{w+1}\right) +  \log \left( 1+\frac{1}{N-w}\right)  - \frac{1}{Np} - \frac{1}{N(1-p)} 
\right]^2  \cdot \frac{1}{\Pr(A)}\\
& =2 t^2_i p(1-p) \EE \left[  \log \left( 1+\frac{1}{w+1}\right) +  \log \left( 1+\frac{1}{N-w}\right)  - \frac{1}{Np(1-p)}
\right]^2  \cdot \frac{1}{\Pr(A)}, \\
% & \stackrel{(d)}{\leq} 2 t^2_i p(1-p) \EE \left[  \left \lvert \frac{1}{w+1} + \frac{1}{N-w}\right)  - \frac{1}{Np(1-p)} \right \rvert + \frac{1}{2}\left \lvert  \frac{1}{(w+1)^2} + \frac{1}{(N-w)^2}\right \rvert 
% \right]^2 \\
\end{align*}
% $(a)$ follows from the fact that only $i^{th}$ term is different. 
where $(a)$ uses the fact that the term is non-zero only if $X_i(j)=1, X'_i(j) = 0$ or $X_i(j)=0, X'_i(j)=1$ and the probability of this event is $2p(1-p)$. $(b)$ uses the fact that for any positive random variable $X$ and any event $A$, $\EE[X]
\geq \Pr(A) \EE[X|A]$. 
%$(d)$ follows from observing that $x - x^2/2 \leq \log (1+x) \leq x $
We first  upper bound the term inside the expectation:
\begin{align*}
& \left( \log \left( 1+\frac{1}{w+1}\right) +  \log \left( 1+\frac{1}{N-w}\right)  - \frac{1}{Np(1-p)}\right)^2 \\
& = \left( \log \left( 1+\frac{1}{w+1}\right) +  \log \left( 1+\frac{1}{N-w} \right) \right)^2 
+ \frac{1}{N^2p^2(1-p)^2} - \\
& \qquad \qquad \qquad \qquad \qquad  \qquad \qquad \frac{2}{Np(1-p)}  \left( \log \left( 1+\frac{1}{w+1}\right) +  \log \left( 1+\frac{1}{N-w}\right) \right) \\
 & \leq \frac{1}{(w+1)^2} + \frac{1}{(N-w)^2} + \frac{2}{w(N-w)} + \frac{1}{N^2p^2(1-p)^2} - \\
 & \qquad \qquad \qquad \qquad \qquad  \qquad \qquad\frac{2}{Np(1-p)}  \left(
 \frac{1}{w+1} - \frac{1}{2(w+1)^2} + \frac{1}{N-w} - \frac{1}{2(N-w)^2}
 \right) \\
 & = \frac{1}{(w+1)(w+2)} - \frac{2}{Np(1-p)} \frac{1}{w+1} + \frac{1}{(N-w)(N-w+1)} - \frac{2}{Np(1-p)} \frac{1}{N-w} \\
 &+ \frac{2}{w(N-w)} \\
& + \frac{1}{(w+1)^2(w+2)} + \frac{1}{(N-w)^2(N-w+1)} + \frac{1}{Np(1-p)} \left( \frac{1}{(w+1)^2} + \frac{1}{(N-w)^2}\right) \\
& + \frac{1}{N^2p^2(1-p)^2},
\end{align*}
where the inequality uses the fact that for any positive $x$, $x - x^2/2 \leq \log x \leq x$. 
Observe that $w \sim \Bin(n-1,p)$ and $N-1-w \sim \Bin(n-1,1-p)$. We use the following three inequalities, to bound the expectation of the term above. Similar results apply for $N-w$ as $N-1-w \sim \Bin(n-1,1-p)$. Since $1/w$ and $1/(N-w)$ are negatively correlated,
\[
\EE\left[ \frac{1}{w(N-w)}\right] \leq 
\EE\left[ \frac{1}{w}\right] \cdot  \EE\left[ \frac{1}{N-w}\right].
\]
Furthermore, for any $i$
\[
\EE\left[ \frac{w!}{(w+i)!} \right]\leq \frac{1}{(Np)^i}
\]
and if $Np(1-p) \geq 2$,
\[
\EE \left[ \frac{1}{(w+1)(w+2)} - \frac{2}{Np(1-p)} \frac{1}{w+1} \right]
\leq 
\frac{1}{(Np)^2} - \frac{2}{N^2p^2(1-p)}.
\]
Combining the above results and simplifying the terms, we get that the expectation of the required quantity is bounded by 
\begin{align*}
% & \frac{1}{N^2p^2} + \frac{1}{N^2(1-p)^2} +\frac{1}{N^2p(1-p)} + \frac{2}{N^2p^2(1-p)^2} - \frac{2}{N^2p^2(1-p)^2} + \frac{3}{N^3p^3} + \frac{3}{N^3(1-p)^3} + 2\frac{1}{Np(1-p)} \left( \frac{1}{N^2p^2} + \frac{1}{N^2(1-p)^2} \right)  \\
& = \frac{1}{N^3p^3(1-p)^3}  \cdot (3p^3 + 3(1-p)^3 + 2p^2 + 2(1-p)^2).
\end{align*}
Hence $\sigma^2_i$ is bounded by 
\[
\frac{1}{\Pr(A)} \cdot  \frac{t^2_i }{N^2p^2(1-p)^2} \cdot (3p^3 + 3(1-p)^3 + 2p^2 + 2(1-p)^2),
\]
and the lemma follows by Bernstein's inequality.
\end{proof}
\begin{proof} [Proof of Theorem \ref{thm:discbinomial}]

Firstly note that it is sufficient to consider the differential privacy of the quantity $\frac{f(D)}{s} + Z$ where $Z$ is a Binomial random variable. Note that since $s$ is defined to be $1/j$ for some integer $j$ the output $f(D)/s$ remains integral. Further note that in this setting the $l_q$ norm sensitivity scales $\Delta_q/s$. The above reduction shows that the scale $s$ can be considered to be 1 in the rest of the proof. 

Consider any two neighboring data sets $D_1, D_2$ and let $\binomdelta \defeq f(D_2) - f(D_1)$. Note that showing the $(\epsilon, \delta)$ differential privacy of the Binomial mechanism is equivalent to showing the following. Let $T$ be a vector such that $T(j) \sim \Bin(N,p)$ then for any vector $v \in [N]^d$ we have that
\[\Pr(T = v) \leq e^{\epsilon}Pr(T = v - \binomdelta) + \delta\]
To show the above we will first define a set $V$ such that 
\[
\Pr(T \in V)\geq 1- \delta,
\]
and for every element $v \in V$, 
\[
\Pr(T = v ) \leq e^{\epsilon}\Pr( T = v - \binomdelta).
\]
Define $V$ as follows: $v \in V$ if and only if,
\begin{gather}
\norm{v - Np}_\infty \leq \beta \triangleq \sqrt{2Np(1-p)\log(20d/\delta)} + \frac{2}{3} \max(p,1-p) \log \frac{20d}{\delta}.\label{eqn:scond1}\\
  | \binomdelta \cdot(v - Np)| \leq \|\binomdelta\|_2 \sqrt{ 2Np(1-p)\log(1.25/\delta)} + \frac{2}{3} \log(1.25/\delta) \|\binomdelta\|_{\infty}.
    \label{eqn:scond2}\\
\forall j, \,\,    v(j) - \binomdelta(j) \in [0, N] 
 \text{ and } v(j) \in Np \pm Np(1-p)/3. \nonumber \\
 \sum^d_{i=1} \Delta(j) \cdot \left( \log \frac{(v(j) + 1)(1-p)}{p(N-v(j)+1)}
- \frac{v(j)+1}{Np} + \frac{N-v(j) + 1}{N(1-p)} \right) \nonumber \leq 
\frac{2\norm{\Delta}_1(p^2+(1-p)^2)}{3 Np(1-p)(1-\delta/10)} \\ \qquad +
\frac{\norm{\Delta}_2 c_p}{Np(1-p)\sqrt{1-\delta/10}}
\cdot \sqrt{ \log \frac{10}{\delta}} + 
\frac{4 \norm{\Delta}_\infty(\beta+1)^2(p^2 + (1-p)^2)}{9N^2p^2(1-p)^2} \log \frac{10}{\delta}.
   \label{eqn:scond3}
\end{gather}
We will first show that the probability of this event is large. 

The first condition follows from Bernstein's inequality with probability $\geq 1- \delta/10$. 
For the second condition, observe that $\binomdelta \cdot(s - Np)$ is a function of $Nd$ independent random variables. 
A direct application of Bernstein's inequality yields that Equation~\eqref{eqn:scond2} holds with probability $\geq 1-\delta/1.25$.
The third condition follows from the first condition as $\norm{\Delta}_\infty\leq Np - \beta$ 
and $ Np(1-p)/3 \geq \beta$. Applying Lemma~\ref{lem:new_log_bound} with $A$ being event that $\norm{v-Np}_\infty \leq \beta$ and $\delta = \delta/10$, yields that the fourth equation holds with probability at least $1 - \delta/10$.
Hence, by the union bound,
\[
\Pr(T \notin V)\leq \delta.
\]
We now prove the ratio of probabilities. For any $v$,
\begin{align*}
& \frac{\Pr( T = v-\Delta)}{\Pr( T = v)} \\
& = \prod^d_{i=1}\frac{\Pr( T(j) = v(j)-\Delta(j))}{\Pr( T(j) = v(j))} \\
& \leq \exp \left(  \sum^d_{i=1} \Delta(j) \cdot \log \frac{(v(j) + 1)(1-p)}{p(N-v(j)+1)}\right) \\
& = \exp \left(
\sum^d_{i=1} \frac{\Delta(j)(v(j)-Np)}{Np(1-p)} + 
\sum^d_{i=1} \Delta(j) \cdot \left( \log \frac{(v(j) + 1)(1-p)}{p(N-v(j)+1)}
- \frac{v(j)+1}{Np} + \frac{N-v(j) + 1}{N(1-p)} \right) \right.\\
&\left. \qquad\qquad\qquad\qquad\qquad\qquad\qquad\qquad\qquad\qquad\qquad\qquad + \frac{\sum^d_{j=1} \Delta(j)(1-2p)}{Np(1-p))}
\right)
\end{align*}
where the inequality follows from Lemma \ref{lemma:probboundbinom}. Since $v\in V$, applying Equations~\eqref{eqn:scond1},~\eqref{eqn:scond2},~\eqref{eqn:scond3}, together with the fact that $\beta \leq \sqrt{2.5 Np(1-p) \log (20d/\delta)}$ (by the assumptions in the theorem) yields the following bound on the exponent. 
\begin{multline*}
\norm{\Delta}_2 
\cdot 
\sqrt{\frac{2 \log \frac{1.25}{\delta}}{Np(1-p)}} + 
\frac{2\norm{\Delta}_\infty }{3Np(1-p)} \log \frac{1.25}{\delta}
+ \frac{\norm{\Delta}_2c_p \sqrt{\log \frac{10}{\delta}}}{Np(1-p)\sqrt{1-\delta/10}}
+ \frac{\norm{\Delta}_\infty d_p \log \frac{20d}{\delta} \log \frac{10}{\delta}}{Np(1-p)}
  \\ + \frac{b_p \|\Delta\|_1}{Np(1-p)(1-\delta/10)},
\end{multline*}
where $c_p$ is defined in Equation~\eqref{eq:cp} and 
\begin{equation}
\label{eq:dp}
d_p \triangleq \frac{4}{3} \cdot (p^2 + (1-p)^2)
\end{equation}
and 
\begin{equation}
\label{eq:bp}
b_p \triangleq   \frac{2(p^2+(1-p)^2)}{3} + (1-2p).
\end{equation}
% \begin{align*}
%  \exp \left( \frac{\norm{\Delta}_2\sqrt{2\log (1.25/\delta)}}{\sqrt{Np(1-p)}}
%  + \frac{2 \norm{\Delta}_\infty \log (1.25/\delta)}{3Np(1-p)} 
% +  \sum^d_{i=1} \frac{5|\Delta(j)| (p^2 + (1-p)^2) \log (10d/\delta)}{3Np(1-p)} \right)
% \\
% = \exp \left( \frac{\norm{\Delta}_2\sqrt{2\log (1.25/\delta)}}{\sqrt{Np(1-p)}}
%  + \frac{2 \norm{\Delta}_\infty \log (1.25/\delta)}{3Np(1-p)} 
% +  \frac{5\norm{\Delta}_1 (p^2 + (1-p)^2) \log (10d/\delta)}{3Np(1-p)} \right)
% % \\
% % \leq \exp \left( \frac{ \norm{\Delta}_2\sqrt{2\log (2/\delta)}}{\sqrt{Np(1-p)}}
% %  + \frac{2 \norm{\Delta}_\infty \log (2/\delta)}{3Np(1-p)} 
% % +  \frac{18\norm{\Delta}_1 \log (2d/\delta)}{Np(1-p)} \right).
% \end{align*}
%Using $\beta $  (assumption in the theorem), results in the theorem.
\end{proof}

\section{High probability sensitivity Proof}
\label{sec:highprobsensitivity}

\begin{proof}
  To show $(\epsilon,\delta + \delta')$ differential privacy we need to show that for any two neighboring data sets $D_1, D_2$ and $O \subseteq \mathcal{O}$, 
  \[\Pr(\mathcal{M}(f(D_1)) \in O) \leq e^{\epsilon}\Pr(\mathcal{M}(f(D_2)) \in O) + \delta + \delta'.\]
  Given any two neighboring data sets $D_1, D_2$ let $\Pr_{\Delta_Q, \delta}(X_1, X_2)$ represent the joint distribution of the coupled random variables $X_1, X_2$ guaranteed by Definition \ref{defn:highprobsense}. Now for any $O \in \mathcal{O}$ we have that 
  \begin{align*}
    Pr(\mathcal{M}(f(D_1)) \in O) &\defeq \int_{s \in \mathcal{S}} \Pr(f(D_1) = s)\Pr(\mathcal{M}(s) \in O) \\ 
    & \stackrel{(a)}{=} \left(\int_{s_1,s_2 | \|s_1 - s_2\|_Q \leq \Delta_Q} \Pr_{\Delta_Q, \delta}(s_1,s_2)(\Pr(\mathcal{M}(s_2) \in O)\right) \\
    &\qquad\qquad\qquad\qquad+ \left(\int_{s_1,s_2 | \|s_1 - s_2\|_Q \geq \Delta_Q} \Pr_{\Delta_Q, \delta}(s_1,s_2)(\Pr(\mathcal{M}(s_2) \in O)\right)\\
    & \stackrel{(b)}{=} \left(\int_{s_1,s_2 | \|s_1 - s_2\|_Q \leq \Delta_Q} \Pr_{\Delta_Q, \delta}(s_1,s_2)(\Pr(\mathcal{M}(s_2) \in O)\right) + \delta\\
    &\stackrel{(c)}{\leq} \left(\int_{s_1,s_2 | \|s_1 - s_2\|_Q \leq \Delta_Q} \Pr_{\Delta_Q, \delta}(s_1,s_2)(e^{\epsilon}\Pr(\mathcal{M}(s_2) \in O) + \delta) \right)+ \delta' \\
    &\stackrel{(d)}{\leq} e^{\epsilon}\left(\int_{s \in \mathcal{S}} \Pr(f(D_2) = s)\Pr(\mathcal{M}(s) \in O) \right) + \delta + \delta' \\
    &\defeq e^{\epsilon}\Pr(\mathcal{M}(f(D_2)) \in O) + \delta + \delta'.
  \end{align*}
    In the above $(a),(d)$ follow from the fact that $\Pr_{\Delta_q, \delta}$ is a coupling, $(b)$ follows from the condition \eqref{eqn:couplingcond} guaranteed by the coupling and $(c)$ follows from the $(\epsilon, \delta)$ differential privacy guarantee of the mechanism $\mathcal{M}$. 
\end{proof}

\section{Application of Binomial Mechanism to Distributed Mean Estimation - Proof of Theorem \ref{thm:MSEbinomial}}
\label{sec:dmebinmechproof}
\begin{proof}[Proof of Theorem \ref{thm:MSEbinomial}]

We refer the readers to the definition of the protocol (Section \ref{sec:BINdescription}) and in particular the definitions of the random variables $\binomquantui_i,\binomnoise_i,$ and the estimator $\binomquantestimator$ given in equations \eqref{eqn:binomquantui} and \eqref{eqn:binomquantestimator} respectively. 

The communication complexity follows immediately by noting that the protocol only transmits integers in the range $[0, k+m)$ and therefore only needs $\log(k+m)$ bits. We now prove the bound on the Mean Square Error of the protocol and then prove the sensitivity guarantee. 
\\
\\
\noindent \textbf{Mean Square Error}
  \begin{align*}
  \norm{\hat{\bar{X}} - \bar{X}}^2_2 &=\frac{1}{n^2}\sum_{j = 1}^{d} \sum_{i = 1}^{n} \expect[(\hat{\bar{X}}_i(j) - X_i(j))^2] \\
    &\leq \frac{1}{n^2}\sum_{j = 1}^{d} \sum_{i = 1}^{n} \expect\left[\left(\frac{2X^{\max}}{k-1}\right)^2\left(\var(\Ber(p_i(j))) + \var(\Bin(mp)) \right)\right] \\
    &\leq (2X^{\max})^2 \left( \frac{d}{4n(k-1)^2} + \frac{d}{n^2}\frac{mnp(1-p)}{(k-1)^2} \right),
  \end{align*}
    where the equality follows from the fact that $\hat{\bar{X}}_i(j)$ are independent of each other and $\hat{\bar{X}}$ is an unbiased estimator of $\hat{\bar{X}}$.
Setting $m,p,k$ as defined in the theorem proves the bound on MSE.
\\
\\
\noindent \textbf{Differential Privacy}
\\
\\
Given two neighboring data sets $X \defeq \{X_1 \ldots X_n\}$ and $X_{\otimes n} \defeq \{X'_1 \ldots X'_n\}$ (where $X'_i = X_i$ for $i \in [1,n-1]$) we will first provide a high probability bound on the $\ell_1, \ell_2, \ell_{\infty}$ sensitivity of quantization protocol $\quantkprotocol$. In particular the following lemma provides the high probability sensitivity bounds. 
\begin{lemma}
\label{lemma:sensitivityproof}
  For every $\delta$, given two neighboring data sets $X \defeq \{X_1 \ldots X_n\}$ and $X_{\otimes n} \defeq \{X'_1 \ldots X'_n\}$ (where $X'_i = X_i$ for $i \in [1,n-1]$) we have that the protocol $\quantkprotocol$ is $(\{\Delta_1, \Delta_2, \Delta_\infty\}, \delta)$-sensitive (c.f. Definition \ref{defn:highprobsense}) where $\Delta_1, \Delta_2, \Delta_{\infty}$ satisfy the following equations.
  \begin{equation}
  \Delta_{\infty} \leq \frac{\|X_n - X'_n\|_{\infty}}{2X^{\max}/(k-1)} + 2
\end{equation}
\begin{equation}
  \Delta_1 \leq \frac{\|X_n - X'_n\|_1}{2X^{\max}/(k-1)} + \sqrt{2\frac{\|X_n - X'_n\|_1 \log(2/\delta)}{2X^{\max}/(k-1)}} + \frac{4}{3}\log(2/\delta) 
\end{equation}
\begin{equation}
  \Delta_2 \leq \frac{\|X_n - X'_n\|_2}{2X^{\max}/(k-1)} + \sqrt{\frac{\|X_n - X'_n\|_1}{2X^{\max}/(k-1)} + \sqrt{\frac{8\|X_n - X'_n\|_1 \log(2/\delta)}{2X^{\max}/(k-1)}}+  \frac{4}{3}\log(2/\delta) }.
\end{equation}
\end{lemma}
Further we note that the protocol $\binomprotocol$ is a composition of the binomial mechanism and the protocol $\quantkprotocol$. A direct application of Theorem \ref{thm:discbinomial} and Lemma \ref{lemma:highprobsensitivity} gives us that the mechanism $\binomprotocol$ is $(\epsilon, 2\delta)$ differentially private for any $\delta \in (0,1)$ and $\epsilon$ satisfying the below conditions. \footnote{we choose $\delta, \delta'$ as $\delta$ in the application of Lemma \ref{lemma:highprobsensitivity}}
% \begin{align}
%  \epsilon &
% = \frac{\Delta_2\sqrt{2\log (1.25/\delta)}}{s\cdot \sqrt{Np(1-p)}}
%  + \frac{2 \Delta_\infty \log (1.25/\delta)}{3\cdot s \cdot  Np(1-p)} \nonumber
% +  \frac{5 \Delta_1 (p^2+(1-p)^2)\log (10d/\delta)}{3s\cdot Np(1-p)}  
% \nonumber \\
% &\geq  \frac{\|X_n - X'_n\|_2\sqrt{2\log(2/\delta)}(k-1)}{\sqrt{mnp(1-p)}*(X^{\max} - X^{\min})} + \sqrt{\frac{16\|X_n - X'_n\|_1\log(2/\delta)(k-1)}{mnp(1-p)*(X^{\max} - X^{\min})}}  \nonumber \\ & \qquad \qquad + \frac{91\|X_n - X'_n\|_1\log(2d/\delta)(k-1)}{mnp(1-p)*(X^{\max} - X^{\min})} + \frac{120 \log(2d/\delta)\log(2/\delta)}{mnp(1-p)}
% \end{align}
Note that the conditions required by Theorem \ref{thm:discbinomial} can be verified from the given conditions in Theorem \ref{thm:MSEbinomial}.
\end{proof}
We now provide a proof of Lemma \ref{lemma:sensitivityproof}. 

\begin{proof}[Proof of Lemma \ref{lemma:sensitivityproof}]

To this end we recall the definition of the random variables $\binomquantui_i(j)$. Given $X^{\max}$ and $X^{\min}$ we associate to every integer $r$ in $[0,k)$ a bin $B(r)$ defined as 

\[B(r) \defeq -X^{\max} + \frac{2rX^{\max}}{k-1}\]

Further given a number $X \in [-X^{\max},X^{\max}]$, let $r(X)$ be the integer such that $X \in [B(r(X)), B(r(X) + 1)]$. We can now define the random variable 
\[
\binomquantui(X) = 
\begin{cases}
    r(X)+1  & \text{w.p. } \frac{X - B(r(X))}{B(r(X) + 1) - B(r(X))} \\
    r(X) & \text{otherwise.}
\end{cases} 
\]

Now define the random variables $\binomquantui^X_i(j) \defeq \binomquantui(X_i(j))$ and similarly $\binomquantui^{X_{\otimes n}}_i(j) \defeq \binomquantui(X'_i(j))$. To provide high probability sensitivity bounds in accordance with Lemma \ref{lemma:highprobsensitivity}, we need to define a coupling between the random variables $\sum_i \binomquantui^X_i$ and $\sum_{i} \binomquantui^{X_{\otimes n}}_i$. To do the above we will define a coupling between the random variables $\binomquantui^X_i(j)$ and $\binomquantui^{X_{\otimes n}}_i(j)$. The coupled random variables will be sampled as follows. 

The defined coupling will have two cases. Define the set $S = \{(i,j) | r(X_i(j)) = r(X'_i(j))\}$. We first consider the case when $(i,j) \in S$. In this case we sample a random variable $\alpha_{ij} \in [0,1]$ uniformly at random and define the random variables

\[Y_{i}(j) = \begin{cases}
    r(X_i(j)) + 1  & \text{if } \alpha_{ij} \leq \frac{X_i(j) - B(r(X_i(j)))}{B(r(X_i(j)) + 1) - B(r(X_i(j)))} \\
    r(X_i(j)) & \text{otherwise.}
\end{cases}\]

\[Y^{\otimes n}_i(j) = \begin{cases}
    r(X'_i(j)) + 1  & \text{if } \alpha_{ij} \leq \frac{X'_i(j) - B(r(X'_i(j)))}{B(r(X'_i(j)) + 1) - B(r(X'_i(j)))} \\
    r(X'_i(j)) & \text{otherwise},
\end{cases}\]

Additionally wlog consider $X_i > X'_i$ (the roles of $i$ and $i'$ can be reversed in the following definitions otherwise) and define the auxiliary variables

\[a_i(j) \defeq \frac{B(r(X_i(j)) + 1) - X_i(j)}{2X^{\max}/(k-1)} \text{ and } b_i(j) \defeq \frac{X'_i(j) - B(r(X'_i(j)))}{2X^{\max}/(k-1)}\]

\[Z_i(j) = \begin{cases}
    0 & \text{w.p. } a_i(j) + b_i(j) \\
    1 & \text{otherwise},
\end{cases}\]

Further define 
\begin{equation}
\label{eqn:aaaa}
  L_{ij} \defeq |Y_i(j) - Y_i^{\otimes n}(j)| = Z_i(j)\;\;\text{ if } \;\; (i,j) \in S
\end{equation}

Otherwise if $(i,j)\notin S$ or equivalently $r(X_i(j)) \neq r(X'_i(j))$, we sample the bins independently and the random variables are defined as

\[Y_{i}(j) = \begin{cases}
    r(X_i(j)) + 1  & \text{w.p. } \frac{X_i(j) - B(r(X_i(j)))}{B(r(X_i(j)) + 1) - B(r(X_i(j)))} \\
    r(X_i(j)) & \text{otherwise.}
\end{cases}\]

\[Y^{\otimes n}_i(j) = \begin{cases}
    r(X'_i(j)) + 1  & \text{w.p. } \frac{X'_i(j) - B(r(X'_i(j)))}{B(r(X'_i(j)) + 1) - B(r(X'_i(j)))} \\
    r(X'_i(j)) & \text{otherwise},
\end{cases}\]

Additionally wlog consider $X_i > X'_i$ (the roles of $i$ and $i'$ can be reversed in the following definitions otherwise) and define the auxiliary variables

\[a_i(j) \defeq \frac{X_i(j) - B(r(X_i(j)))}{2X^{\max}/(k-1)} \text{ and } b_i(j) \defeq \frac{B(r(X'_i(j)) + 1)- X'_i(j)}{2X^{\max} /(k-1)}\]

\[Z_i(j) = \begin{cases}
    0 & \text{w.p. } 1 - a_i(j) - b_i(j) + a_i(j)b_i(j)\\
    1 & \text{w.p. } a_i(j) + b_i(j) - 2a_i(j)b_i(j) \\
    2 & \text{otherwise},
\end{cases}\]

In this case define $L_{i,j} \defeq r(X_i(j)) - r(X'_i(j)) + 1 + Z_i(j)$ and note that
\begin{equation}
\label{eqn:aaab}
  |Y_i(j) - Y_i^{\otimes n}(j)| \leq L_{ij} 
\end{equation}

With these definitions, it can be seen that the marginal distributions of $Y_{i}(j), Y^{\otimes n}_i(j)$ are equal to the marginal distributions of $\binomquantui^X_i(j)$, $\binomquantui^{X_{\otimes n}}_i(j)$ respectively. Further note that since $X'_i = X_i$ for all $i \in [1,n-1]$ we have that $Y_i = Y^{\otimes n}_i$ w.p. 1 for all $i \in [1,n-1]$. Therefore
\[ \| \sum_i Y_i - \sum_i Y^{\otimes n}_i \|_q^q = \|Y_n - Y^{\otimes n}_n\|_q^q \leq \sum_{j} L_{nj}^q,\]
where the inequality follows from \eqref{eqn:aaaa} and \eqref{eqn:aaab}. We wish to bound the RHS above. To that end consider the following claim which follows from the definitions. 
\begin{claim}
\label{claim:aaaa}
  \[Z_{i}(j) \leq 2 \quad \text{w.p. 1}\]
  \[\expect[Z_i(j)] = \begin{cases}
    a_i(j) + b_i(j) & \text{if } (i,j) \notin S\\
    1 - (a_i(j) + b_i(j)) & \text{otherwise}
    \end{cases}\]
    \[\expect[Z_i(j) - \expect[Z_i(j)]^2] \leq \begin{cases}
    a_i(j) + b_i(j) & \text{if } (i,j) \notin S\\
    1 - (a_i(j) + b_i(j)) & \text{otherwise}
    \end{cases} = \expect[Z_i(j)]\]
    \[\expect[Z_i(j) - \expect[Z_i(j)]^4] \leq 4\expect[Z_i(j) - \expect[Z_i(j)]^2] \leq 4  \expect[Z_i(j)].\]
\end{claim}
Further note that
\begin{equation}
  \label{eqn:aaac}
  \sum_j \expect[Z_n(j)] = \sum_{(n,j) \notin S} (a_i(j) + b_i(j)) + \sum_{(n,j) \in S} 1 - (a_i(j) + b_i(j)) \leq \frac{\|X_n - X'_n\|_1}{2X^{\max}/(k-1)}.
\end{equation}
A direct application of Bernstein's Inequality gives us that with probability at least $1 - \delta/2$
\begin{equation}
  \label{eqn:aaad}
  \sum_j  Z_n(j)  \leq \expect[\sum_j Z_n(j)] + \sqrt{2 \expect[\sum_j Z_n(j)] \log(2/\delta)} + \frac{4}{3}\log(2/\delta).
\end{equation}
This gives us that 
\begin{align*}
  \sum_j |Y_n(j) - Y^{\otimes}_n(j)| &\stackrel{a}{\leq} \sum_j L_{nj} \\
  & \stackrel{b}{\leq}\sum_{(i,j) \in S} Z_i(j) + \sum_{(i,j) \notin S} \left( r(X_i(j)) - r(X'_i(j)) + 1 + Z_i(j)\right) \\ 
  &\stackrel{c}{\leq} \frac{\|X_n - X'_n\|_1}{2X^{\max}/(k-1)} + \sqrt{2 \frac{\|X_n - X'_n\|_1}{2X^{\max} /(k-1)} \log(2/\delta)} + \frac{4}{3} \log(2/\delta)
\end{align*}

where $a,b$ follow from \eqref{eqn:aaaa} and \eqref{eqn:aaab} and $c$ follows from Claim \ref{claim:aaaa} and \eqref{eqn:aaac}. This proves the $\ell_1$ norm bound. 

We now focus on the $\ell_2$ norm case. For this we note that

\[ \forall(i,j) \;\;L_{ij} = \begin{cases} 
\frac{X_i(j) - X'_{i}(j)}{2X^{\max}/(k-1)} + Z_{i}(j) - \expect[Z_i(j)] & \text{if } X_i(j) \geq X'_{i}(j) \\
\frac{X'_i(j) - X_{i}(j)}{2X^{\max} /(k-1)} + Z_{i}(j) - \expect[Z_i(j)] & \text{if } X_i(j) < X'_{i}(j).
\end{cases}
 \] 

Therefore 
\begin{equation}
\label{eqn:aaae}
  \sqrt{\sum_{j} L^2_{nj}} = \sqrt{\sum_j \left(\frac{X_i(j) - X'_{i}(j)}{2X^{\max} /(k-1)}\right)^2} + \sqrt{\sum_j (Z_{n}(j) - \expect Z_{n}(j))^2}.
\end{equation}

We now bound $\sqrt{\sum_j (Z_{n}(j) - \expect Z_{n}(j))^2}$. We can now apply Bernstein's inequality on the random variable $(Z_{n}(j) - \expect Z_{n}(j))^2$ to get that with probability at least $1 - \delta/2$
\begin{equation}
\label{eqn:aaaf}
  \sum_j (Z_{n}(j) - \expect Z_{n}(j))^2 \leq \sum_j E[Z_{nj}] + \sqrt{8 \sum_j E[Z_{nj}] \log(2/\delta)} + \frac{4}{3}  \log(2/\delta),
\end{equation}

where the RHS uses Claim \ref{claim:aaaa} for bounding expectation and variance.

Therefore combining \eqref{eqn:aaae} and \eqref{eqn:aaaf}, we get that
\begin{align*}
  &\|Y_n - Y'_n\|_2 \leq \sqrt{\sum_{j} L^2_{nj}} \\
  &\leq \frac{\|X_n - X'_{n}\|_2}{2X^{\max}/(k-1)} + \sqrt{ \frac{\|X_n - X'_{n}\|_1}{2X^{\max}/(k-1)} + \sqrt{8 \left(\frac{\|X_n - X'_{n}\|_1}{2X^{\max}/(k-1)}\right) \log(2/\delta)} + \frac{4}{3}  \log(2/\delta)}.
\end{align*}

The proof is finished using a union bound.

\end{proof}

\section{Quantization with Rotation}
\label{sec:rotateddmebinmechproof}

We prove Theorem~\ref{thm:mainthmbinomialrotated} here.

% The rest of the section is devoted to the proof of the above theorem. 

\noindent \textbf{Differential Privacy}
\\
Given any two neighboring data sets $X = \{X_1, \ldots X_n\},X_{\otimes n}= \{X_1, \ldots X'_n\}$ we define a set of good rotations $U_{\text{good}}$ as follows 
\[ U_{\text{good}} = \left\{ R \in U | \forall\;i \in [n]\;\;\; \|RX_i\|_{\infty} \leq \frac{2\sqrt{\log(\frac{2nd}{\delta}})\bound_2}{\sqrt{d}}, \|RX'_n\|_{\infty} \leq \frac{2\sqrt{\log(\frac{2nd}{\delta}})\bound_2}{\sqrt{d}}\right\}\]
where $U$ is the set of $d \times d$ orthonormal matrices. The following lemma follows from \cite{AilonL09}. We note that similar analysis holds for uniformly sampled $R$ over real domain and we refer the reader to \cite{dasgupta2003elementary} for details. 
\begin{lemma}[\cite{AilonL09}]
\[P(HA \in U_{\text{good}}) \geq 1 - \delta\] 
\end{lemma}

Let $\Rot(\pi, HA)(X), \Rot(\pi, HA)(X_{\otimes n})$ represent the random output of the protocol $\Rot(\pi, HA)$ on $X, X_{\otimes n}$ respectively and let $S$ be any subset of the output range of $\Rot(\pi, HA)$. Given $\delta$ let $\epsilon$ be given by Theorem \ref{thm:discbinomial} with sensitivity parameters $\{\Delta_1(X^{max}, D), \Delta_2(X^{max}, D), \Delta_\infty(X^{max}, D)\}$. Given a set of vectors $V$ and a rotation matrix $R$ define $R \cdot V = \{Rv | v \in V\}$.  
\begin{align*}
  &Pr(\Rot(\pi, HA)(X) \in S) \\
  &\leq \int_{R \in U_{\text{good}}}\left(Pr(\Rot(\binomprotocol, HA)(X) \in S | R)\right) dR + Pr(R \notin U_{\text{good}}) \\
  &=  \int_{R \in U_{\text{good}}}Pr(\Rot(\binomprotocol, HA)(R\cdot X) \in R\cdot S)dR + Pr(R \notin U_{\text{good}}) \\
  &\stackrel{a}{\leq}  \int_{R \in U_{\text{good}}}\left(e^{\epsilon}Pr(\binomprotocol(R\cdot X_{\otimes n}) \in R\cdot S) + 2\delta \right)dR + Pr(R \notin U_{\text{good}}) \\
  &=  \int_{R \in U_{\text{good}}}e^{\epsilon}\left(Pr(\Rot(\binomprotocol, HA)(X_{\otimes n}) \in S | R) + 2\delta \right) dR + Pr(R \notin U_{\text{good}}) \\
  &\leq e^{\epsilon} Pr(\Rot(\binomprotocol, HA)(X_{\otimes n}) \in S) + 3\delta
\end{align*}

$a$ follows from $(\epsilon, 2\delta)$ differential privacy guarantee for $\binomprotocol$ from Theorem \ref{thm:MSEbinomial} and noting that $R \in U_{good}$ in the integral. Hence $\Rot(\binomprotocol)$ offers $(\epsilon, 3 \delta)$ differential-privacy.

\noindent \textbf{Mean Square Error}
\\
The bound on the MSE can be observed by noting that the total change the entire protocol can cause on any individual client vector is bounded by $2\bound$ in $\ell_2$ norm, therefore the total MSE can be at most $4\bound^2$ irrespective of the choice of rotation. 
Therefore
\begin{align*}
  \cE(\Rot(\pi_{sk}(\Bin(m,p))), HA) &= \cE(\Rot(\pi_{sk}(\Bin(m,p))), HA | R \in U_{good}) + \\
  & \qquad\qquad\qquad\qquad\cE(\Rot(\pi_{sk}(\Bin(m,p))), HA | R \notin U_{good})\\
  &\stackrel{a}{\leq} \cE(\Rot(\pi_{sk}(\Bin(m,p))), HA | R \in U_{good}) + 4\bound^2\delta^2 \\
  &\stackrel{b}{\leq}\frac{2 \log \frac{2nd}{\delta} \cdot D^2}{n(k-1)^2} + \frac{8 \log \frac{2nd}{\delta}}{n}\cdot\frac{mp(1-p) D^2}{(k-1)^2} + 4\bound^2\delta^2
\end{align*}
$a$ follows from the argument above and $b$ follows from the MSE guarantee in Theorem \ref{thm:MSEbinomial} and by noting that the rotation is in $U_{good}$. 

\end{document}